\documentclass{article}

\usepackage{arxiv}

\usepackage[utf8]{inputenc}
\usepackage[T1]{fontenc}
\usepackage{hyperref}
\usepackage{url}
\usepackage{booktabs}
\usepackage{amsmath,amssymb,amsfonts}
\usepackage{mathrsfs}
\usepackage{nicefrac}
\usepackage{microtype}
\usepackage{graphicx}
\usepackage{multirow}
\usepackage{xcolor}
\usepackage{textcomp}
\usepackage{algorithm}
\usepackage{algorithmicx}
\usepackage{algpseudocode}
\usepackage{listings}
\usepackage{tikz}
\usetikzlibrary{arrows.meta}
\usepackage{enumitem}
\usepackage{caption}
\usepackage{float}
\usepackage[export]{adjustbox}
\usepackage{pgfplots}
\pgfplotsset{compat=1.18}
\usepgfplotslibrary{groupplots}
\pgfplotsset{
    resultsgrid/.style={
        every axis/.append style={
            ybar, bar width=4pt,
            width=0.5\textwidth, height=0.4\textwidth,
            symbolic x coords={Class A, Class B, Class C, Class D, Class E, Class F, Class G, Overall},
            xtick=data,
            xticklabel style={rotate=60, anchor=east, font=\tiny},
            tick label style={font=\tiny},
            label style={font=\scriptsize},
            title style={font=\small, yshift=-2mm},
            enlarge x limits=0.15,
            legend style={font=\small}, legend columns=-1,
        }
    }
}
\usepackage{longtable}
\usepackage{array}
\usepackage{rotating}
\usepackage{cleveref}

\definecolor{pieceRed}{RGB}{234,6,6}
\newlength{\PolyUnit}
\tikzset{
  poly/.style={x=\PolyUnit,y=\PolyUnit}
}

\usepackage{amsthm}
\theoremstyle{plain}
\newtheorem{theorem}{Theorem}

\newtheorem{corollary}[theorem]{Corollary}

\theoremstyle{definition}

\newtheorem{remark}{Remark}
\newtheorem{definition}{Definition}

\newcommand{\hgeneral}{h_{\mathrm{general}}}

\title{Class-Based Heuristic Selection for Solving the Flying Block Puzzle}

\author{
  Sanyar Ahmadi \\
  Faculty of Computer Engineering\\
  University of Tehran\\
  Tehran, Iran\\
  \texttt{sanyar.ahmadi@ut.ac.ir} \\
  \And
  Pedram Asadzadeh \\
  Faculty of Computer Engineering\\
  K.~N.~Toosi University of Technology\\
  Tehran, Iran\\
  \texttt{p.asadzadeh@email.kntu.ac.ir} \\
  \And
  Amanj Khorramian\thanks{Corresponding author.} \\
  Department of Electrical and Computer Engineering\\
  University of Kurdistan\\
  Sanandaj, Iran\\
  \texttt{a.khorramian@uok.ac.ir} \\
}

\renewcommand{\shorttitle}{Class-Based Heuristic Selection for the Flying Block Puzzle}

\hypersetup{
  pdftitle={Class-Based Heuristic Selection for Solving the Flying Block Puzzle},
  pdfauthor={Sanyar Ahmadi, Pedram Asadzadeh, Amanj Khorramian},
  pdfkeywords={Flying Block Puzzle, NP-complete, Heuristic Search Algorithm,
    Depth-Prioritized A*, Class-Based Heuristic Selection, Spatial Planning},
}

\begin{document}
\maketitle

\begin{abstract}
Heuristic search underlies planning in autonomous systems ranging from warehouse
logistics to robotic navigation, yet generic heuristics fail to exploit the
structural constraints that govern constrained spatial domains, causing search
performance to degrade catastrophically on harder instances. We study this
problem through the two-column Flying Block Puzzle, a rigorously NP-complete
spatial planning microworld whose bottleneck geometry mirrors clearance-to-size
constraints encountered in multi-agent path finding, autonomous vehicle
navigation, and block relocation systems. We introduce the Class-Based Heuristic
A* (CBHA*) algorithm, which integrates a General Move Constraint to capture
minimum displacement costs when vacant units are scarce, a formal kinematic
taxonomy partitioning the state space into seven mutually exclusive classes with
provably admissible heuristics based on vacancy ratio and goal-piece geometry,
and a class-conditional tie-breaking mechanism that dynamically switches between
depth-priority and vertical-distance ordering to overcome $f$-value plateaus.
Over 146 benchmark instances, CBHA* achieves a 93.4\% success rate against 64\%
for Depth-Prioritized A*, 39\% for Standard A*, and 17\% for BFS, while reducing
node expansions by 87.98\% relative to Standard A* and sustaining an average
effective branching factor of approximately 3, demonstrating that class-triggered
adaptive heuristics constitute a principled mechanism for efficient spatial
planning that generalizes structurally to physical constraint systems.
\end{abstract}

\keywords{Flying Block Puzzle \and NP-complete \and Heuristic Search Algorithm
\and Depth-Prioritized A* \and Class-Based Heuristic Selection \and Spatial Planning}

\section{Introduction}

Combinatorial puzzles serve as canonical AI microworlds~\cite{russell2020artificial}
for heuristic search in planning and navigation.  As Uehara~\cite{uehara2026first} argues, well-designed
puzzles are exceptionally suited for explaining fundamental algorithmic
techniques, since a robust algorithm inherently possesses a refined beauty
analogous to the satisfaction of solving a complex puzzle.  This perspective
is reinforced throughout Knuth's \emph{The Art of Computer Programming}
\cite{knuth2022art}, which uses classical recreational challenges---the Soma
cube, pentominoes, word searches---to elucidate advanced search and
backtracking algorithms, observing that playful experimentation with such
problems often reveals mathematical structures that suggest
``useful cutoff strategies and/or data structures''~\cite{knuth2022art}
for real-world software.
The 2-column Flying Block variant is rigorously NP-complete
\cite{kanzaki2024computational} and provides a tractable yet challenging
testbed for spatial algorithm design
\cite{mccarthy1990chess, ensmenger2012chess, hearn2005pspace}. Unlike purely logical NP-complete benchmarks, this puzzle requires reasoning about geometric occlusion and cascading positional dependencies. The bottleneck structure of this variant---where limited vacant cells force long sliding sequences---mirrors constraints in multi-agent path finding, warehouse robotics, and block relocation \cite{lin2022review, yang2023path, caserta2012mathematical}. Puzzle pieces are polyominoes placed in a $2\times h$ frame; piece shapes are formally classified in Section~\ref{sec:formulation}.

A* search \cite{hart1968formal} balances optimal pathfinding with heuristic guidance, but its performance depends critically on heuristic quality. Generic heuristics fail to exploit the structural properties specific to individual puzzle variants and suffer from scalability problems on larger instances \cite{mathew2013experimental, korf1996finding, sanchez2020systematic}. Our Class-Based Heuristic Selection (CBHS) framework addresses this by instantiating the algorithm selection paradigm \cite{russell2020artificial,kotthoff2015algorithm,brazdil2017metalearning}, dynamically mapping instance features to optimal heuristic strategies. Our contributions are:

\begin{enumerate}
	\item \textbf{Puzzle-Specific General Heuristic:} An admissible, consistent heuristic derived from the General Move Constraint (Theorem~\ref{thm:general_move}), providing improved guidance toward optimal solutions.

	\item \textbf{Depth-Prioritized A* (DPA*):} Validating depth-based tie-breaking \cite{asai2016tiebreaking} within this domain and using it as the foundation for the class-conditional tie-breaker in Contribution~4.

	\item \textbf{Class-Based Heuristic Selection:} A formal kinematic taxonomy partitioning the state space into seven mutually exclusive, collectively exhaustive classes over $k \ge 1$ (where $k$ denotes the number of vacant units), each admitting a consistent heuristic (Theorem~\ref{thm:partition}).

	\item \textbf{Class-Conditional Tie-Breaker Selection:} A mechanism that dynamically selects depth or vertical-distance tie-breaking per class, further optimizing search performance.
\end{enumerate}

By formulating class-tailored admissible heuristics for this NP-complete
spatial microworld, this work advances the research tradition that
\cite{uehara2026first} and \cite{knuth2022art} identify as a primary engine
for algorithmic insight, using constrained spatial planning as a lens for
broader principles in autonomous robotics, navigation, logistics, and
multi-agent pathfinding.

\section{Problem Formulation and Properties}\label{sec:formulation}

Let $F = \{1,2\} \times \{1,\ldots,h\}$ be a rectangular frame of width~2 and height~$h$. Each piece $p_i$ is a polyomino occupying cells $C_P \subset F$; the target piece $P_t$ must be moved to a designated goal position. In Sections~\ref{subsec:class-c}--\ref{subsec:class-e}, we also write matrix coordinates $(r,c)$ with row $r$ increasing downward. In all figures, gray cells represent occupied spaces and white-colored cells indicate vacant units.

Piece shapes are polyominoes classified by linear symmetry (Definition~\ref{def:linear-symmetry}) into four exhaustive types (Figure~\ref{fig:piece_types}):

\begin{definition}[Linear symmetry]\label{def:linear-symmetry}
	A shape $S\subset\mathbb{Z}^2$ is \emph{linearly symmetric} if there exists a reflection $\rho$ of $\mathbb{R}^2$ (an isometry of order~2 fixing a line pointwise) such that $\rho(S)=S$.
\end{definition}

\begin{definition}[I-shaped piece]\label{def:i-shaped}
	A piece is \textbf{I-shaped with length $\ell$} if its shape equals $S^I_\ell = \{0\}\times\{0,\dots,\ell-1\}$ up to translation: a $1\times\ell$ rectangle occupying a single column.
\end{definition}

\begin{definition}[L-shaped piece]\label{def:l-shaped}
	A piece is \textbf{L-shaped} if its shape equals some $D_4$-image of the right tromino $S^L=\{(0,0),(1,0),(0,1)\}$---the unique non-rectangular 3-cell polyomino.
\end{definition}

\begin{definition}[Comb-shaped piece]\label{def:comb-shaped}
	A piece is \textbf{comb-shaped} with body height $H$ and $s$ free spaces if $C_P = \text{Body}\cup\text{Teeth}$, where $\text{Body} = \{x_0\}\times\{y_0,\dots,y_0+H-1\}$, $\text{Teeth} = \{x_1\}\times T$ with $\emptyset \neq T \subsetneq \{y_0,\dots,y_0+H-1\}$, and $\text{Free} = \{x_1\}\times(\{y_0,\dots,y_0+H-1\}\setminus T)$ with $s := |\text{Free}|$, subject to linear symmetry about the horizontal midline (Figure~\ref{fig:piece_types}(d)).
\end{definition}

\begin{definition}[Other linearly symmetric pieces]\label{def:other-type}
	A shape is of \textbf{other} type if it is linearly symmetric and not I-shaped, L-shaped, or comb-shaped (Figure~\ref{fig:piece_types}(e)).
\end{definition}

\begin{figure}[htbp]
	\centering
	\includegraphics[width=\textwidth]{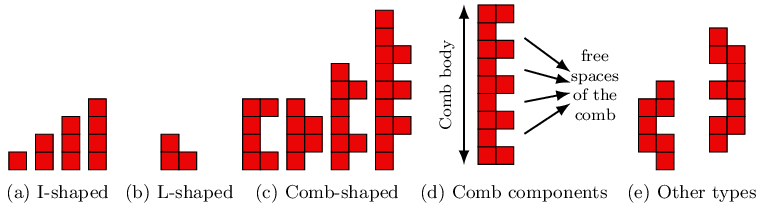}
	\caption{Categorization of polyomino pieces in the Flying Block puzzle.}
	\label{fig:piece_types}
\end{figure}

The objective is to find a minimum-length move sequence $\pi(s_0, s_g)$ placing $P_t$ at its designated target; all moves have unit cost, so $f^*(s_0) = \min_\pi N$ where $N$ is the sequence length. We distinguish \emph{jumping} moves ($C'_P \cap C_P = \emptyset$) from \emph{sliding} moves ($C'_P \cap C_P \neq \emptyset$, $|C'_P \cap C_P| < |C_P|$). This variant is NP-complete \cite{kanzaki2024computational}.

\begin{theorem}[General Move Constraint]
	\label{thm:general_move}
	Let $P$ be a piece of size $n$ with $k$ vacant cells available. Any valid move satisfies
	$|C_P \cap C'_P| \ge n - k$. Intuitively, a piece can vacate at most $k$ cells per move, since only vacant cells can be newly occupied (proof in Appendix~\ref{app:sn1-proofs}).
\end{theorem}

\begin{corollary}[Jumping Move Constraint]\label{cor:jumping}
	If $k < n$, piece $P$ cannot execute a jumping move (proof in Appendix~\ref{app:sn1-proofs}).
\end{corollary}

\begin{corollary}[Crossing Constraint]\label{cor:crossing}
	A rectangular piece of size $n > 2$ with $k < n$ cannot cross from one column to the other through any sequence of moves (proof in Appendix~\ref{app:sn1-proofs}).
\end{corollary}

Three solvability consequences follow: (I)~any piece with $k < n$ is restricted to sliding moves and cannot execute a jumping move (Corollary~\ref{cor:jumping}; Figure~\ref{fig:impossibility}(a)); (II)~a rectangular piece with $n > 2$ and $k < n$ cannot reach the opposite column (Corollary~\ref{cor:crossing}; Figure~\ref{fig:impossibility}(b)); (III)~a comb-shaped piece with $k < s$ cannot flip across its horizontal axis of symmetry (Theorem~\ref{thm:comb_flip}; Figure~\ref{fig:impossibility}(c)).

\begin{theorem}[Comb-Piece Flip Constraint]\label{thm:comb_flip}
	Let $P$ be a comb-shaped piece with $s:=|\text{Free}|$ free spaces and $k$ vacant cells. If $k<s$, then no sequence of valid moves can flip $P$ across its horizontal axis of symmetry. Intuitively, the flip would require the $s$ free spaces to become newly occupied, but only $k < s$ cells are available (proof in Appendix~\ref{app:sn2-comb-flip}).
\end{theorem}

\begin{figure}[htbp]
	\centering
	\includegraphics[width=.8\textwidth]{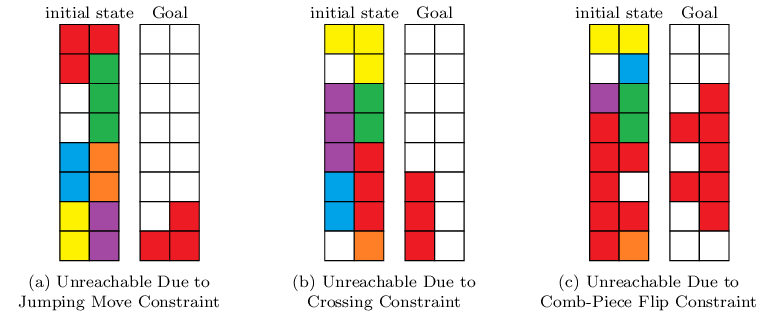}
	\caption{Unsolvable instances illustrating the three solvability constraints
	of Section~\ref{sec:formulation}. Each panel shows an initial state (left)
	and the target goal position (right).
	(a)~Unreachable due to the Jumping Move Constraint
	(Corollary~\ref{cor:jumping}): both the target piece and the blocking piece
	have size greater than the number of vacant units, so neither can execute a
	jumping move and the target piece cannot bypass the blocker to reach the goal.
	(b)~Unreachable due to the Crossing Constraint
	(Corollary~\ref{cor:crossing}): the target piece is an I-shaped piece of
	size~3 with fewer than three vacant units; it cannot cross to the opposite
	column where the goal position is located.
	(c)~Unreachable due to the Comb-Piece Flip Constraint
	(Theorem~\ref{thm:comb_flip}): the target piece is a comb-shaped piece with
	three free spaces, but the number of vacant units is strictly less than three,
	so the comb cannot flip across its horizontal axis of symmetry to align with
	the goal.}
	\label{fig:impossibility}
\end{figure}

\section{Heuristic Development for A* Algorithm}\label{sec:heuristics}

\subsection{Primitive Heuristic}

A baseline \emph{primitive heuristic} $h_p(s)=|P(s)|+\delta(s)$, where
$P(s)$ is the set of non-target pieces occupying at least one goal cell and
$\delta(s)\in\{0,1\}$ indicates whether $s$ is a non-goal state, provides
an admissible, consistent lower bound (Appendix~\ref{app:sn3-primitive}).  The general
heuristic below refines this estimate for the constrained sliding regime.

\subsection{General Heuristic}

When a blocking piece $p_i$ occupies $g_i$ goal cells and there are $k$ vacancies, Theorem~\ref{thm:general_move} implies at least $\lceil g_i / k \rceil$ moves are needed to clear it---motivating the following refinement (full derivation in Appendix~\ref{app:sn4-general}).

\begin{definition}[General Heuristic Function]
	\label{def:general_heuristic}
	Let $s$ be a state in the puzzle's state space. Let $G_{pos}$ be the set of cells that must be occupied by the target piece in the goal state. Let $P(s)$ be the set of non-target pieces that occupy at least one cell in $G_{pos}$ in state $s$. For each $p_i \in P(s)$, let $g_i$ be the number of goal cells in $G_{pos}$ occupied by $p_i$.

	Let $k$ denote the number of vacant units in the puzzle. The general heuristic function $h(s)$ is defined as
	\(
		h(s) = \sum_{p_i \in P(s)} \left\lceil \frac{g_i}{k} \right\rceil + \delta(s)
	\)
	where
	\[
		\delta(s) =
		\begin{cases}
			1 & \text{if } s \text{ is not a goal state} \\
			0 & \text{if } s \text{ is a goal state}.
		\end{cases}
	\]
\end{definition}

Figure~\ref{fig:heuristic-motivating-examples}(a)--(b) illustrates this heuristic.

\begin{figure}[p]
\centering
\includegraphics[width=\textwidth]{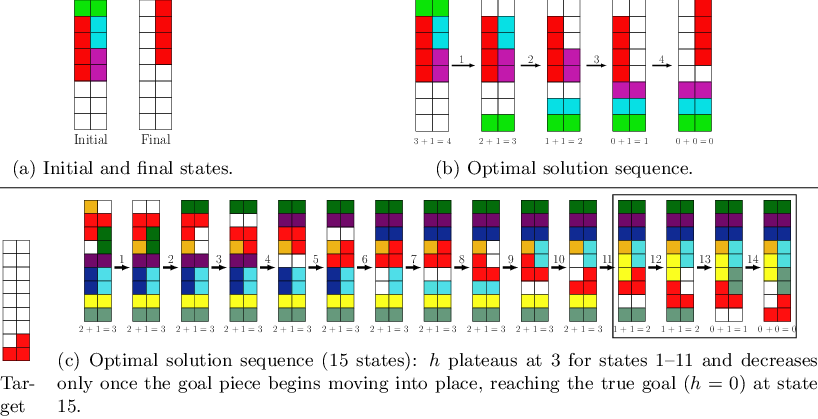}
\caption{Worked examples motivating the general and class-based heuristics.
\textbf{(a--b)} General heuristic: initial/final states and
the optimal-solution sequence (Section~\ref{sec:heuristics}, General Heuristic).
\textbf{(c)} Motivating example
for class-based heuristic selection: target state and the full 15-state
optimal sequence, showing the general heuristic plateaued at $h=3$ across
states~1--11 before decreasing over the final four states to the true goal
at state~15.}
\label{fig:heuristic-motivating-examples}
\end{figure}

\begin{theorem}[Consistency of $\hgeneral(s)$]
	\label{thm:general_consistency}
	$\hgeneral$ is consistent: a single move displaces at most $k$ cells of any blocking piece, reducing $\lceil g_p/k \rceil$ by at most $1$ (proof in Appendix~\ref{app:sn4-general}).
\end{theorem}

\begin{algorithm}[H]
\caption{Heuristic Calculation for the General Heuristic}
\label{alg:general-heuristic}
\begin{algorithmic}[1]
\Require state $s$, goal $g$
\Ensure heuristic\_value
\State goal\_occupied\_positions $\leftarrow$ \{ $(i, j) \mid goal[i][j] \neq 0$ \} \Comment{Find goal positions}
\State pieces\_occupying\_goal $\leftarrow$ \{ board[$i$][$j$] $\mid$ $(i, j) \in$ goal\_occupied\_positions \textbf{and} board[$i$][$j$] $\neq 0$ \textbf{and} board[$i$][$j$] $\neq$ goal\_piece \} \Comment{Find non-target pieces occupying goal}
\State heuristic\_value $\leftarrow 0$ \Comment{Initialize the heuristic value}
\ForAll{$p \in$ pieces\_occupying\_goal}
    \State num\_goal\_positions\_occupied $\leftarrow$ number of $(i, j) \in$ goal\_occupied\_positions \textbf{where} board[$i$][$j$] $= p$
    \State heuristic\_value $\leftarrow$ heuristic\_value $+$ $\left\lceil \dfrac{\text{num\_goal\_positions\_occupied}}{\text{num\_vacant\_units}} \right\rceil$
\EndFor
\If{\textbf{not} is\_goal(board, goal)} \Comment{Add 1 if not goal state}
    \State heuristic\_value $\leftarrow$ heuristic\_value $+ 1$
\EndIf
\State \Return heuristic\_value
\end{algorithmic}
\end{algorithm}

\section{Depth-Prioritized A* Search}\label{sec4}

The general heuristic often assigns the same $f$-value to many frontier nodes (e.g., with 3 goal positions, $h$ ranges only from 0 to 4), leaving A* with insufficient guidance. Depth-based tie-breaking within $f$-value plateaus is well-established \cite{asai2016tiebreaking,pereira2016improved,correa2018analyzing}. Ranking tied nodes by greater $g$ is equivalent to smaller $h$, recovering the classical convention \cite{asai2016tiebreaking}; our contribution is the class-conditional extension in Section~\ref{subsec:tiebreaker-selection}.

Among several tie-breaking criteria tested (misplaced units, piece-to-goal distance, search depth), depth-priority performed best: when $f=g+h$ ties, we expand deeper nodes first. Figure~\ref{fig:dpa-walkthrough} shows a $6\times 2$ instance and illustrates how depth-first expansion resolves the large $f$-value plateau. Algorithm~\ref{alg:priority-comparison} implements this comparison; here depth is the root-to-node distance $g(n)$, which differs from the plateau-relative metric of Asai and Fukunaga \cite{asai2016tiebreaking}.

\begin{algorithm}[H]
\caption{Priority comparison function. Nodes with lower $f$ are preferred; ties are broken by greater $g$ (depth from root). See text for distinction from Asai and Fukunaga \cite{asai2016tiebreaking}.}
\label{alg:priority-comparison}
\begin{algorithmic}[1]
\Require node $a$, node $b$
\Ensure node with higher priority
\If{$a.f\_value < b.f\_value$}
    \State \Return $a$
\ElsIf{$a.f\_value > b.f\_value$}
    \State \Return $b$
\ElsIf{$a.f\_value = b.f\_value$}
    \If{$a.depth > b.depth$}
        \State \Return $a$
    \Else
        \State \Return $b$
    \EndIf
\EndIf
\end{algorithmic}
\end{algorithm}

\begin{figure}[!htbp]
	\centering
	\resizebox{0.72\textwidth}{!}{
		\begin{tikzpicture}

			\tikzstyle{every node}=[font=\LARGE]

			\draw [ fill={rgb,255:red,249; green,6; blue,6} ] (8.75,11.25) rectangle (7.5,10);
			\draw [ fill={rgb,255:red,249; green,6; blue,6} ] (10,11.25) rectangle (8.75,10);
			\draw [ fill={rgb,255:red,249; green,6; blue,6} ] (7.5,10) rectangle (8.75,8.75);
			\draw [ fill={rgb,255:red,233; green,250; blue,0} ] (10,10) rectangle (8.75,8.75);
			\draw [ fill={rgb,255:red,194; green,25; blue,172} ] (7.5,8.75) rectangle (8.75,7.5);
			\draw [ fill={rgb,255:red,194; green,25; blue,172} ] (10,8.75) rectangle (8.75,7.5);
			\draw [ fill={rgb,255:red,11; green,228; blue,7} ] (7.5,7.5) rectangle (8.75,6.25);
			\draw [ fill={rgb,255:red,7; green,224; blue,228} ] (10,7.5) rectangle (8.75,6.25);
			\draw [ fill={rgb,255:red,240; green,179; blue,45} ] (7.5,6.25) rectangle (8.75,5);
			\draw [ fill={rgb,255:red,0; green,42; blue,255} ] (10,6.25) rectangle (8.75,5);
			\draw  (7.5,5) rectangle (8.75,3.75);
			\draw  (10,5) rectangle (8.75,3.75);
			\draw [->, line width=2mm, >=latex] (7.25,7.5) -- (-1.5,16.75);
			\draw [ fill={rgb,255:red,249; green,6; blue,6} ] (-3.75,21.25) rectangle (-5,20);
			\draw [ fill={rgb,255:red,249; green,6; blue,6} ] (-2.5,21.25) rectangle (-3.75,20);
			\draw [ fill={rgb,255:red,249; green,6; blue,6} ] (-5,20) rectangle (-3.75,18.75);
			\draw  (-2.5,20) rectangle (-3.75,18.75);
			\draw [ fill={rgb,255:red,194; green,25; blue,172} ] (-5,18.75) rectangle (-3.75,17.5);
			\draw [ fill={rgb,255:red,194; green,25; blue,172} ] (-2.5,18.75) rectangle (-3.75,17.5);
			\draw [ fill={rgb,255:red,11; green,228; blue,7} ] (-5,17.5) rectangle (-3.75,16.25);
			\draw [ fill={rgb,255:red,7; green,224; blue,228} ] (-2.5,17.5) rectangle (-3.75,16.25);
			\draw [ fill={rgb,255:red,240; green,179; blue,45} ] (-5,16.25) rectangle (-3.75,15);
			\draw [ fill={rgb,255:red,0; green,42; blue,255} ] (-2.5,16.25) rectangle (-3.75,15);
			\draw [ fill={rgb,255:red,233; green,250; blue,0} ] (-5,15) rectangle (-3.75,13.75);
			\draw  (-2.5,15) rectangle (-3.75,13.75);
			\draw [->, line width=2mm, >=latex] (-5.5,17.5) -- (-10,22.5);
			\draw [->, line width=2mm, >=latex] (-5.5,17.5) -- (-9.75,14.5);
			\draw [ fill={rgb,255:red,249; green,6; blue,6} ] (-12.5,26.25) rectangle (-13.75,25);
			\draw [ fill={rgb,255:red,249; green,6; blue,6} ] (-11.25,26.25) rectangle (-12.5,25);
			\draw [ fill={rgb,255:red,249; green,6; blue,6} ] (-13.75,25) rectangle (-12.5,23.75);
			\draw [ fill={rgb,255:red,7; green,224; blue,228} ] (-11.25,25) rectangle (-12.5,23.75);
			\draw [ fill={rgb,255:red,194; green,25; blue,172} ] (-13.75,23.75) rectangle (-12.5,22.5);
			\draw [ fill={rgb,255:red,194; green,25; blue,172} ] (-11.25,23.75) rectangle (-12.5,22.5);
			\draw [ fill={rgb,255:red,11; green,228; blue,7} ] (-13.75,22.5) rectangle (-12.5,21.25);
			\draw  (-11.25,22.5) rectangle (-12.5,21.25);
			\draw [ fill={rgb,255:red,240; green,179; blue,45} ] (-13.75,21.25) rectangle (-12.5,20);
			\draw [ fill={rgb,255:red,0; green,42; blue,255} ] (-11.25,21.25) rectangle (-12.5,20);
			\draw [ fill={rgb,255:red,233; green,250; blue,0} ] (-13.75,20) rectangle (-12.5,18.75);
			\draw  (-11.25,20) rectangle (-12.5,18.75);
			\draw [ fill={rgb,255:red,249; green,6; blue,6} ] (-12.5,17.5) rectangle (-13.75,16.25);
			\draw [ fill={rgb,255:red,249; green,6; blue,6} ] (-11.25,17.5) rectangle (-12.5,16.25);
			\draw [ fill={rgb,255:red,194; green,25; blue,172} ] (-13.75,16.25) rectangle (-12.5,15);
			\draw [ fill={rgb,255:red,194; green,25; blue,172} ] (-11.25,16.25) rectangle (-12.5,15);
			\draw [ fill={rgb,255:red,11; green,228; blue,7} ] (-13.75,13.75) rectangle (-12.5,12.5);
			\draw [ fill={rgb,255:red,7; green,224; blue,228} ] (-11.25,13.75) rectangle (-12.5,12.5);
			\draw [ fill={rgb,255:red,240; green,179; blue,45} ] (-13.75,12.5) rectangle (-12.5,11.25);
			\draw [ fill={rgb,255:red,0; green,42; blue,255} ] (-11.25,12.5) rectangle (-12.5,11.25);
			\draw [ fill={rgb,255:red,233; green,250; blue,0} ] (-13.75,11.25) rectangle (-12.5,10);
			\draw  (-11.25,11.25) rectangle (-12.5,10);
			\draw [ color={rgb,255:red,255; green,51; blue,0} , line width=1.1pt ] (-15,8.75) rectangle (-1.75,26.75);
			\node [font=\Huge] at (-4,13) {\scalebox{1.3}{3 + 1 = 4}};
			\node [font=\Huge] at (-12.5,18.25) {\scalebox{1.3}{2 + 2 = 4}};
			\node [font=\Huge] at (-12.5,9.5) {\scalebox{1.3}{2 + 2 = 4}};
			\draw [->, line width=2mm, >=latex] (7.25,7.5) -- (-3.5,7);
			\draw [->, line width=2mm, >=latex] (7.25,7.5) -- (-3.25,5.25);
			\draw [ color={rgb,255:red,249; green,6; blue,6} , line width=1.1pt ] (-4.25,7.75) rectangle (-7.5,6.75);
			\draw [ fill={rgb,255:red,0; green,0; blue,0} , line width=1.1pt ] (-5,7.25) circle (0.25cm);
			\draw [ fill={rgb,255:red,0; green,0; blue,0} , line width=1.1pt ] (-6,7.25) circle (0.25cm);
			\draw [ fill={rgb,255:red,0; green,0; blue,0} , line width=1.1pt ] (-7,7.25) circle (0.25cm);
			\draw [ color={rgb,255:red,249; green,6; blue,6} , line width=1.1pt ] (-4.25,5.75) rectangle (-7.5,4.75);
			\draw [ fill={rgb,255:red,0; green,0; blue,0} , line width=1.1pt ] (-5,5.25) circle (0.25cm);
			\draw [ fill={rgb,255:red,0; green,0; blue,0} , line width=1.1pt ] (-6,5.25) circle (0.25cm);
			\draw [ fill={rgb,255:red,0; green,0; blue,0} , line width=1.1pt ] (-7,5.25) circle (0.25cm);
			\draw [->, line width=2mm, >=latex] (7.25,7.5) -- (-1.25,0.25);
			\draw [ fill={rgb,255:red,249; green,6; blue,6} ] (-3.75,3.75) rectangle (-5,2.5);
			\draw [ fill={rgb,255:red,249; green,6; blue,6} ] (-2.5,3.75) rectangle (-3.75,2.5);
			\draw [ fill={rgb,255:red,249; green,6; blue,6} ] (-5,2.5) rectangle (-3.75,1.25);
			\draw [ fill={rgb,255:red,233; green,250; blue,0} ] (-2.5,2.5) rectangle (-3.75,1.25);
			\draw [ fill={rgb,255:red,194; green,25; blue,172} ] (-5,1.25) rectangle (-3.75,0);
			\draw [ fill={rgb,255:red,194; green,25; blue,172} ] (-2.5,1.25) rectangle (-3.75,0);
			\draw [ fill={rgb,255:red,11; green,228; blue,7} ] (-5,0) rectangle (-3.75,-1.25);
			\draw [ fill={rgb,255:red,7; green,224; blue,228} ] (-2.5,0) rectangle (-3.75,-1.25);
			\draw [ fill={rgb,255:red,240; green,179; blue,45} ] (-5,-1.25) rectangle (-3.75,-2.5);
			\draw  (-2.5,-1.25) rectangle (-3.75,-2.5);
			\draw [ fill={rgb,255:red,0; green,42; blue,255} ] (-5,-2.5) rectangle (-3.75,-3.75);
			\draw  (-2.5,-2.5) rectangle (-3.75,-3.75);
			\draw [->, line width=2mm, >=latex] (-5.5,-2.25) -- (-11,1.25);
			\draw [->, line width=2mm, >=latex] (-5.5,-2.25) -- (-11,-7.25);
			\draw [ fill={rgb,255:red,249; green,6; blue,6} ] (-13.75,3.75) rectangle (-15,2.5);
			\draw [ fill={rgb,255:red,249; green,6; blue,6} ] (-12.5,3.75) rectangle (-13.75,2.5);
			\draw [ fill={rgb,255:red,249; green,6; blue,6} ] (-15,2.5) rectangle (-13.75,1.25);
			\draw [ fill={rgb,255:red,233; green,250; blue,0} ] (-12.5,2.5) rectangle (-13.75,1.25);
			\draw [ fill={rgb,255:red,194; green,25; blue,172} ] (-15,1.25) rectangle (-13.75,0);
			\draw [ fill={rgb,255:red,194; green,25; blue,172} ] (-12.5,1.25) rectangle (-13.75,0);
			\draw [ fill={rgb,255:red,11; green,228; blue,7} ] (-15,0) rectangle (-13.75,-1.25);
			\draw  (-12.5,0) rectangle (-13.75,-1.25);
			\draw [ fill={rgb,255:red,240; green,179; blue,45} ] (-15,-1.25) rectangle (-13.75,-2.5);
			\draw [ fill={rgb,255:red,7; green,224; blue,228} ] (-12.5,-1.25) rectangle (-13.75,-2.5);
			\draw [ fill={rgb,255:red,0; green,42; blue,255} ] (-15,-2.5) rectangle (-13.75,-3.75);
			\draw  (-12.5,-2.5) rectangle (-13.75,-3.75);
			\draw [ fill={rgb,255:red,249; green,6; blue,6} ] (-13.75,-5) rectangle (-15,-6.25);
			\draw [ fill={rgb,255:red,249; green,6; blue,6} ] (-12.5,-5) rectangle (-13.75,-6.25);
			\draw [ fill={rgb,255:red,249; green,6; blue,6} ] (-15,-6.25) rectangle (-13.75,-7.5);
			\draw [ fill={rgb,255:red,233; green,250; blue,0} ] (-12.5,-6.25) rectangle (-13.75,-7.5);
			\draw [ fill={rgb,255:red,194; green,25; blue,172} ] (-15,-7.5) rectangle (-13.75,-8.75);
			\draw [ fill={rgb,255:red,194; green,25; blue,172} ] (-12.5,-7.5) rectangle (-13.75,-8.75);
			\draw [ fill={rgb,255:red,11; green,228; blue,7} ] (-15,-8.75) rectangle (-13.75,-10);
			\draw  (-12.5,-8.75) rectangle (-13.75,-10);
			\draw [ fill={rgb,255:red,240; green,179; blue,45} ] (-15,-10) rectangle (-13.75,-11.25);
			\draw  (-12.5,-10) rectangle (-13.75,-11.25);
			\draw [ fill={rgb,255:red,0; green,42; blue,255} ] (-15,-11.25) rectangle (-13.75,-12.5);
			\draw [ fill={rgb,255:red,7; green,224; blue,255} ] (-12.5,-11.25) rectangle (-13.75,-12.5);
			\node [font=\Huge] at (-4.25,-4.5) {\scalebox{1.4}{3 + 1 = 4}};
			\node [font=\Huge] at (-14,-4.25) {\scalebox{1.4}{2 + 2 = 4}};
			\node [font=\Huge] at (-13.75,-13) {\scalebox{1.4}{2 + 2 = 4}};
			\draw [->, line width=2mm, >=latex] (10.25,7.5) -- (23.75,23.75);
			\draw [ fill={rgb,255:red,249; green,6; blue,6} ] (26.5,27.25) rectangle (25.25,26);
			\draw [ fill={rgb,255:red,249; green,6; blue,6} ] (27.75,27.25) rectangle (26.5,26);
			\draw [ fill={rgb,255:red,249; green,6; blue,6} ] (25.25,26) rectangle (26.5,24.75);
			\draw [ fill={rgb,255:red,233; green,250; blue,0} ] (27.75,26) rectangle (26.5,24.75);
			\draw [ fill={rgb,255:red,194; green,25; blue,172} ] (25.25,24.75) rectangle (26.5,23.5);
			\draw [ fill={rgb,255:red,194; green,25; blue,172} ] (27.75,24.75) rectangle (26.5,23.5);
			\draw [ fill={rgb,255:red,11; green,228; blue,7} ] (25.25,23.5) rectangle (26.5,22.25);
			\draw [ fill={rgb,255:red,7; green,224; blue,228} ] (27.75,23.5) rectangle (26.5,22.25);
			\draw  (25.25,22.25) rectangle (26.5,21);
			\draw [ fill={rgb,255:red,0; green,42; blue,255} ] (27.75,22.25) rectangle (26.5,21);
			\draw  (25.25,21) rectangle (26.5,19.75);
			\draw [ fill={rgb,255:red,240; green,179; blue,45} ] (27.75,21) rectangle (26.5,19.75);
			\draw [->, line width=2mm, >=latex] (28,23.5) -- (34.5,27.75);
			\draw [->, line width=2mm, >=latex] (28,23.5) -- (34.75,20);
			\draw [ fill={rgb,255:red,249; green,6; blue,6} ] (36.25,32) rectangle (35,30.75);
			\draw [ fill={rgb,255:red,249; green,6; blue,6} ] (37.5,32) rectangle (36.25,30.75);
			\draw [ fill={rgb,255:red,249; green,6; blue,6} ] (35,30.75) rectangle (36.25,29.5);
			\draw [ fill={rgb,255:red,233; green,250; blue,0} ] (37.5,30.75) rectangle (36.25,29.5);
			\draw [ fill={rgb,255:red,194; green,25; blue,172} ] (35,29.5) rectangle (36.25,28.25);
			\draw [ fill={rgb,255:red,194; green,25; blue,172} ] (37.5,29.5) rectangle (36.25,28.25);
			\draw [ fill={rgb,255:red,11; green,228; blue,7} ] (35,28.25) rectangle (36.25,27);
			\draw  (37.5,28.25) rectangle (36.25,27);
			\draw  (35,27) rectangle (36.25,25.75);
			\draw [ fill={rgb,255:red,0; green,42; blue,255} ] (37.5,27) rectangle (36.25,25.75);
			\draw [ fill={rgb,255:red,7; green,224; blue,255} ] (35,25.75) rectangle (36.25,24.5);
			\draw [ fill={rgb,255:red,240; green,179; blue,45} ] (37.5,25.75) rectangle (36.25,24.5);
			\draw [ fill={rgb,255:red,249; green,6; blue,6} ] (36.25,22.75) rectangle (35,21.5);
			\draw [ fill={rgb,255:red,249; green,6; blue,6} ] (37.5,22.75) rectangle (36.25,21.5);
			\draw [ fill={rgb,255:red,249; green,6; blue,6} ] (35,21.5) rectangle (36.25,20.25);
			\draw [ fill={rgb,255:red,233; green,250; blue,0} ] (37.5,21.5) rectangle (36.25,20.25);
			\draw [ fill={rgb,255:red,194; green,25; blue,172} ] (35,20.25) rectangle (36.25,19);
			\draw [ fill={rgb,255:red,194; green,25; blue,172} ] (37.5,20.25) rectangle (36.25,19);
			\draw [ fill={rgb,255:red,11; green,228; blue,7} ] (35,19) rectangle (36.25,17.75);
			\draw  (37.5,19) rectangle (36.25,17.75);
			\draw [ fill={rgb,255:red,7; green,224; blue,228} ] (35,17.75) rectangle (36.25,16.5);
			\draw [ fill={rgb,255:red,0; green,42; blue,255} ] (37.5,17.75) rectangle (36.25,16.5);
			\draw  (35,16.5) rectangle (36.25,15.25);
			\draw [ fill={rgb,255:red,240; green,179; blue,45} ] (37.5,16.5) rectangle (36.25,15.25);
			\node [font=\Huge] at (26.75,19) {\scalebox{1.4}{3 + 1 = 4}};
			\node [font=\Huge] at (36,23.75) {\scalebox{1.4}{2 + 2 = 4}};
			\node [font=\Huge] at (36,14.5) {\scalebox{1.4}{2 + 2 = 4}};
			\draw [->, line width=2mm, >=latex] (10.25,7.5) -- (21.5,7.25);
			\draw [ fill={rgb,255:red,249; green,6; blue,6} ] (25.25,8.75) rectangle (24,7.5);
			\draw [ fill={rgb,255:red,249; green,6; blue,6} ] (26.5,8.75) rectangle (25.25,7.5);
			\draw [ fill={rgb,255:red,249; green,6; blue,6} ] (24,7.5) rectangle (25.25,6.25);
			\draw [ fill={rgb,255:red,233; green,250; blue,0} ] (26.5,7.5) rectangle (25.25,6.25);
			\draw  (24,6.25) rectangle (25.25,5);
			\draw  (26.5,6.25) rectangle (25.25,5);
			\draw [ fill={rgb,255:red,11; green,228; blue,7} ] (24,5) rectangle (25.25,3.75);
			\draw [ fill={rgb,255:red,7; green,224; blue,228} ] (26.5,5) rectangle (25.25,3.75);
			\draw [ fill={rgb,255:red,240; green,179; blue,45} ] (24,3.75) rectangle (25.25,2.5);
			\draw [ fill={rgb,255:red,0; green,42; blue,255} ] (26.5,3.75) rectangle (25.25,2.5);
			\draw [ fill={rgb,255:red,194; green,25; blue,172} ] (24,2.5) rectangle (25.25,1.25);
			\draw [ fill={rgb,255:red,194; green,25; blue,172} ] (26.5,2.5) rectangle (25.25,1.25);
			\draw [->, line width=2mm, >=latex] (27,5) -- (28.25,5);
			\draw  (30.25,8.75) rectangle (29,7.5);
			\draw  (31.5,8.75) rectangle (30.25,7.5);
			\draw [ fill={rgb,255:red,249; green,6; blue,6} ] (29,7.5) rectangle (30.25,6.25);
			\draw [ fill={rgb,255:red,233; green,250; blue,0} ] (31.5,7.5) rectangle (30.25,6.25);
			\draw [ fill={rgb,255:red,249; green,6; blue,6} ] (29,6.25) rectangle (30.25,5);
			\draw [ fill={rgb,255:red,249; green,6; blue,6} ] (31.5,6.25) rectangle (30.25,5);
			\draw [ fill={rgb,255:red,11; green,228; blue,7} ] (29,5) rectangle (30.25,3.75);
			\draw [ fill={rgb,255:red,7; green,224; blue,228} ] (31.5,5) rectangle (30.25,3.75);
			\draw [ fill={rgb,255:red,240; green,179; blue,45} ] (29,3.75) rectangle (30.25,2.5);
			\draw [ fill={rgb,255:red,0; green,42; blue,255} ] (31.5,3.75) rectangle (30.25,2.5);
			\draw [ fill={rgb,255:red,194; green,25; blue,172} ] (29,2.5) rectangle (30.25,1.25);
			\draw [ fill={rgb,255:red,194; green,25; blue,172} ] (31.5,2.5) rectangle (30.25,1.25);
			\draw [->, line width=2mm, >=latex] (32,5) -- (33.5,5);
			\draw  (35.25,8.75) rectangle (34,7.5);
			\draw [ fill={rgb,255:red,7; green,224; blue,228} ] (36.5,8.75) rectangle (35.25,7.5);
			\draw [ fill={rgb,255:red,249; green,6; blue,6} ] (34,7.5) rectangle (35.25,6.25);
			\draw [ fill={rgb,255:red,233; green,250; blue,0} ] (36.5,7.5) rectangle (35.25,6.25);
			\draw [ fill={rgb,255:red,249; green,6; blue,6} ] (34,6.25) rectangle (35.25,5);
			\draw [ fill={rgb,255:red,249; green,6; blue,6} ] (36.5,6.25) rectangle (35.25,5);
			\draw [ fill={rgb,255:red,11; green,228; blue,7} ] (34,5) rectangle (35.25,3.75);
			\draw  (36.5,5) rectangle (35.25,3.75);
			\draw [ fill={rgb,255:red,240; green,179; blue,45} ] (34,3.75) rectangle (35.25,2.5);
			\draw [ fill={rgb,255:red,0; green,42; blue,255} ] (36.5,3.75) rectangle (35.25,2.5);
			\draw [ fill={rgb,255:red,194; green,25; blue,172} ] (34,2.5) rectangle (35.25,1.25);
			\draw [ fill={rgb,255:red,194; green,25; blue,172} ] (36.5,2.5) rectangle (35.25,1.25);
			\draw [->, line width=2mm, >=latex] (37,5) -- (38.5,5);
			\draw  (40.25,8.75) rectangle (39,7.5);
			\draw [ fill={rgb,255:red,7; green,224; blue,228} ] (41.5,8.75) rectangle (40.25,7.5);
			\draw  (39,7.5) rectangle (40.25,6.25);
			\draw [ fill={rgb,255:red,233; green,250; blue,0} ] (41.5,7.5) rectangle (40.25,6.25);
			\draw [ fill={rgb,255:red,249; green,6; blue,6} ] (39,6.25) rectangle (40.25,5);
			\draw [ fill={rgb,255:red,249; green,6; blue,6} ] (41.5,6.25) rectangle (40.25,5);
			\draw [ fill={rgb,255:red,11; green,228; blue,7} ] (39,5) rectangle (40.25,3.75);
			\draw [ fill={rgb,255:red,249; green,6; blue,6} ] (41.5,5) rectangle (40.25,3.75);
			\draw [ fill={rgb,255:red,240; green,179; blue,45} ] (39,3.75) rectangle (40.25,2.5);
			\draw [ fill={rgb,255:red,0; green,42; blue,255} ] (41.5,3.75) rectangle (40.25,2.5);
			\draw [ fill={rgb,255:red,194; green,25; blue,172} ] (39,2.5) rectangle (40.25,1.25);
			\draw [ fill={rgb,255:red,194; green,25; blue,172} ] (41.5,2.5) rectangle (40.25,1.25);
			\node [font=\Huge] at (25.25,0.5) {\scalebox{1.4}{2 + 1 = 3}};
			\node [font=\Huge] at (30.25,0.5) {\scalebox{1.4}{2 + 2 = 4}};
			\node [font=\Huge] at (35.25,0.5) {\scalebox{1.4}{1 + 3 = 4}};
			\node [font=\Huge] at (40.25,0.5) {\scalebox{1.4}{0 + 4 = 4}};
			\draw [->, line width=2mm, >=latex] (10.25,7.5) -- (19.75,-2.5);
			\draw [ fill={rgb,255:red,249; green,6; blue,6} ] (22.75,-4.25) rectangle (21.5,-5.5);
			\draw [ fill={rgb,255:red,249; green,6; blue,6} ] (24,-4.25) rectangle (22.75,-5.5);
			\draw [ fill={rgb,255:red,249; green,6; blue,6} ] (21.5,-5.5) rectangle (22.75,-6.75);
			\draw [ fill={rgb,255:red,233; green,250; blue,0} ] (24,-5.5) rectangle (22.75,-6.75);
			\draw [ fill={rgb,255:red,194; green,25; blue,172} ] (21.5,-6.75) rectangle (22.75,-8);
			\draw [ fill={rgb,255:red,194; green,25; blue,172} ] (24,-6.75) rectangle (22.75,-8);
			\draw [ fill={rgb,255:red,11; green,228; blue,7} ] (21.5,-8) rectangle (22.75,-9.25);
			\draw [ fill={rgb,255:red,7; green,224; blue,228} ] (24,-8) rectangle (22.75,-9.25);
			\draw  (21.5,-9.25) rectangle (22.75,-10.5);
			\draw [ fill={rgb,255:red,0; green,42; blue,255} ] (24,-9.25) rectangle (22.75,-10.5);
			\draw [ fill={rgb,255:red,240; green,179; blue,45} ] (21.5,-10.5) rectangle (22.75,-11.75);
			\draw  (24,-10.5) rectangle (22.75,-11.75);
			\draw [->, line width=2mm, >=latex] (24.25,-8.75) -- (29.75,-4.25);
			\draw [->, line width=2mm, >=latex] (24.25,-8.75) -- (28.5,-14.25);
			\draw [ fill={rgb,255:red,249; green,6; blue,6} ] (31.25,-2.5) rectangle (30,-3.75);
			\draw [ fill={rgb,255:red,249; green,6; blue,6} ] (32.5,-2.5) rectangle (31.25,-3.75);
			\draw [ fill={rgb,255:red,249; green,6; blue,6} ] (30,-3.75) rectangle (31.25,-5);
			\draw [ fill={rgb,255:red,233; green,250; blue,0} ] (32.5,-3.75) rectangle (31.25,-5);
			\draw [ fill={rgb,255:red,194; green,25; blue,172} ] (30,-5) rectangle (31.25,-6.25);
			\draw [ fill={rgb,255:red,194; green,25; blue,172} ] (32.5,-5) rectangle (31.25,-6.25);
			\draw [ fill={rgb,255:red,11; green,228; blue,7} ] (30,-6.25) rectangle (31.25,-7.5);
			\draw  (32.5,-6.25) rectangle (31.25,-7.5);
			\draw  (30,-7.5) rectangle (31.25,-8.75);
			\draw [ fill={rgb,255:red,0; green,42; blue,255} ] (32.5,-7.5) rectangle (31.25,-8.75);
			\draw [ fill={rgb,255:red,240; green,179; blue,45} ] (30,-8.75) rectangle (31.25,-10);
			\draw [ fill={rgb,255:red,7; green,224; blue,228} ] (32.5,-8.75) rectangle (31.25,-10);
			\draw [ fill={rgb,255:red,249; green,6; blue,6} ] (31.25,-11.75) rectangle (30,-13);
			\draw [ fill={rgb,255:red,249; green,6; blue,6} ] (32.5,-11.75) rectangle (31.25,-13);
			\draw [ fill={rgb,255:red,249; green,6; blue,6} ] (30,-13) rectangle (31.25,-14.25);
			\draw [ fill={rgb,255:red,233; green,250; blue,0} ] (32.5,-13) rectangle (31.25,-14.25);
			\draw [ fill={rgb,255:red,194; green,25; blue,172} ] (30,-14.25) rectangle (31.25,-15.5);
			\draw [ fill={rgb,255:red,194; green,25; blue,172} ] (32.5,-14.25) rectangle (31.25,-15.5);
			\draw [ fill={rgb,255:red,11; green,228; blue,7} ] (30,-15.5) rectangle (31.25,-16.75);
			\draw  (32.5,-15.5) rectangle (31.25,-16.75);
			\draw [ fill={rgb,255:red,7; green,224; blue,228} ] (30,-16.75) rectangle (31.25,-18);
			\draw [ fill={rgb,255:red,0; green,42; blue,255} ] (32.5,-16.75) rectangle (31.25,-18);
			\draw [ fill={rgb,255:red,240; green,179; blue,45} ] (30,-18) rectangle (31.25,-19.25);
			\draw  (32.5,-18) rectangle (31.25,-19.25);
			\node [font=\Huge] at (22.75,-12.75) {\scalebox{1.4}{3 + 1 = 4}};
			\node [font=\Huge] at (31.25,-10.75) {\scalebox{1.4}{2 + 2 = 4}};
			\node [font=\Huge] at (31.25,-20.25) {\scalebox{1.4}{2 + 2 = 4}};
			\draw [ color={rgb,255:red,249; green,6; blue,6} , line width=1.1pt ] (-1.75,4.5) rectangle (-16.25,-14);
			\draw [ color={rgb,255:red,249; green,6; blue,6} , line width=1.1pt ] (38.5,32.5) rectangle (24.25,13.25);
			\draw [ color={rgb,255:red,249; green,6; blue,6} , line width=1.1pt ] (33.5,-2) rectangle (20.5,-21.25);
			\node [font=\Huge] at (32.75,10.25) {\scalebox{1.6}{Optimal solution}};
			\draw [ color={rgb,255:red,11; green,228; blue,7} , line width=1.1pt ] (42.75,12.5) rectangle (22,-1.25);
			\node [font=\Huge] at (8.5,12.25) {\scalebox{1.6}{Initial state}};

		\end{tikzpicture}
	}

	\caption{Illustration of a puzzle solved in 4 moves, highlighting the optimal path. The figure shows how depth prioritization improves search efficiency by prioritizing deeper states with the same $f$-value, allowing the algorithm to focus on more promising paths and reach the goal faster.}
	\label{fig:dpa-walkthrough}

\end{figure}
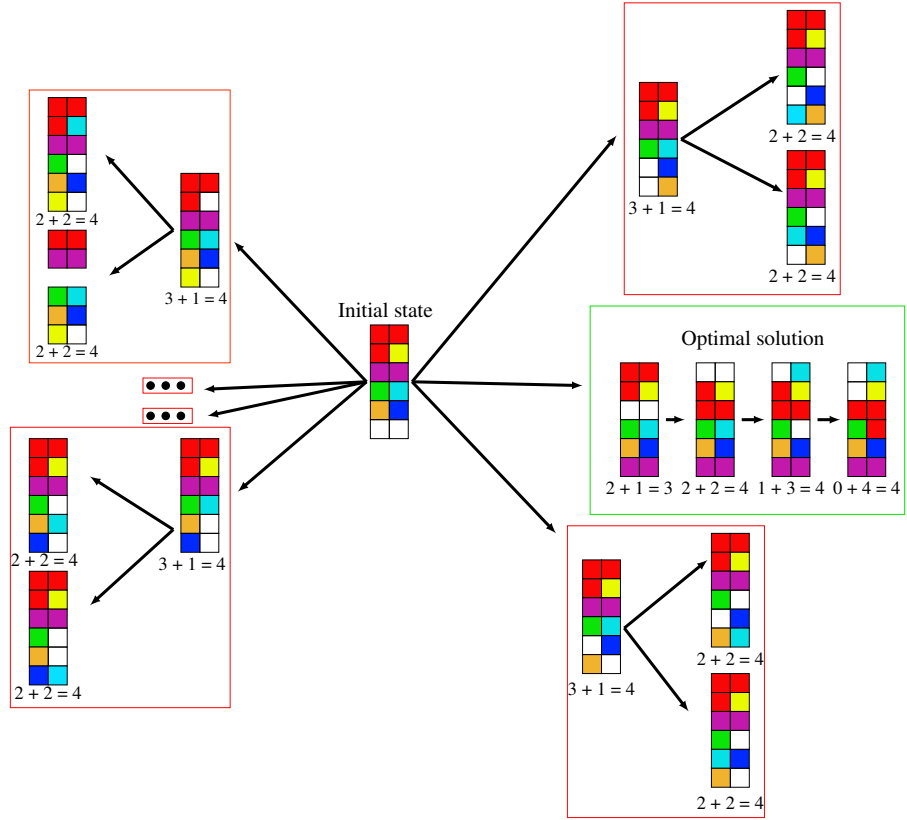

\section{Kinematic Taxonomy and State-Space Partitioning}\label{sec5}

The general heuristic $\hgeneral$ plateaus at $h = 3$ across the first eleven of fifteen states on the sliding-regime motivating instance of Figure~\ref{fig:heuristic-motivating-examples}(c), providing no discriminating guidance until the final four steps. This motivates a kinematic partition of the state space by the pair (vacancy ratio $k/n$, goal-piece geometry), from which class-specific admissible heuristics are derived. For each class, we designed a specialized consistent heuristic; CBHA* detects the class from the initial configuration and applies the corresponding heuristic, using graph-search mode with duplicate detection for all classes (Section~\ref{sec:methods}).

\subsection{Sufficiency: The Partition Tree}\label{subsec:partition-sufficiency}

We consider only states with $k \ge 1$ (at $k = 0$ no move is possible). Algorithm~\ref{alg:state-type-assignment} implements the partition tree; each branch is forced by a named result:
(1)~Theorem~\ref{thm:general_move} separates $k \ge n$ (\textbf{Class~A}, jumping available) from $k < n$ (sliding only).
(2)~Definitions~\ref{def:i-shaped}--\ref{def:other-type} partition shapes into I, L, comb, and other.
(3)~For I-shaped: Corollary~\ref{cor:crossing} with $n \ge 3$ yields \textbf{Class~B}; the remaining case $n=2$ yields \textbf{Class~C} (rotation enables column-crossing).
(4)~For L-shaped ($n=3$): $k=1$ yields \textbf{Class~D}; $k=2$ yields \textbf{Class~E} (size-1 blockers can be bypassed).
(5)~For comb-shaped: Theorem~\ref{thm:comb_flip} with $s > k$ pins the body to one side $\rightarrow$ \textbf{Class~F}; otherwise $\rightarrow$ \textbf{Class~G}.
(6)~``Other'' shapes also fall to \textbf{Class~G}.
Thus Class~G is reached by two independent routes, as reflected in the \texttt{Else} branch of Algorithm~\ref{alg:state-type-assignment}.

\paragraph{Necessity of the seven-class partition.}
The partition into exactly seven classes is not merely a design choice but a
formal necessity.  Certain class boundaries are admissibility-critical: applying
a neighbouring class's heuristic across a boundary produces inadmissible
estimates, as the borrowed rule fails to account for the kinematic freedoms
specific to the target class.  Other boundaries arise from guidance rather than
admissibility considerations, since the general heuristic remains consistent
across the entire state space yet becomes uninformative in the sliding-only
regime.

\begin{theorem}[Kinematic partition]\label{thm:partition}
	For puzzle states with $k \ge 1$, Algorithm~\ref{alg:state-type-assignment} partitions the state space into exactly seven classes A--G. For each class $X$, there exists a heuristic $h_X$ that is admissible and consistent over every state in $X$ (Theorem~\ref{thm:general_consistency} for A; Theorems~\ref{thm:hB_consistency}--\ref{thm:hE_consistency} for B--E; Sections~\ref{subsec:class-f} and~\ref{subsec:class-g} for F and~G).
\end{theorem}

\begin{remark}[Class A and informativeness]\label{rem:class-a-informativeness}
	Merging Class~A with any other class does not yield inadmissibility for $h_{\mathrm{general}}$---Theorem~\ref{thm:general_consistency} proves consistency on the entire state space. Class~A is nevertheless distinguished because $h_{\mathrm{general}}$ becomes uninformative once $k<n$ (Figure~\ref{fig:heuristic-motivating-examples}(c)), so retaining Class~A is a \emph{guidance} necessity rather than an admissibility necessity.
\end{remark}

\subsection{Class Assignment Algorithm}\label{subsec:class-detection}

\begin{algorithm}[H]
\caption{State Type Assignment Algorithm. Each branch corresponds to a split
in the kinematic partition ($k \ge 1$; $n$ = goal-piece size; $s = |\text{Free}|$):
$k \ge n \Rightarrow$ Class~A; I-shaped with $n \ge 3 \Rightarrow$ Class~B;
$n = 2 \Rightarrow$ Class~C; L-shaped with $k=1 \Rightarrow$ Class~D, $k=2
\Rightarrow$ Class~E; comb-shaped with $s>k \Rightarrow$ Class~F; remaining
cases (comb with $s \le k$, or other symmetric shape) $\Rightarrow$ Class~G
via two independent routes.
For I-shaped pieces, $n \ge 3$ yields Class~B; $n = 2$ yields Class~C (since the enclosing condition enforces $k < n = 2$, and $k \ge 1$ together force $k = 1$, enabling rotation and column-crossing). The case $n = 1$ is excluded: a single-cell I-shaped piece always satisfies $k \ge n = 1$, placing any such state in Class~A by the outer branch.}
\label{alg:state-type-assignment}
\begin{algorithmic}[1]
\Require Number of vacant units $num\_of\_vacants$, Goal piece size $goal\_piece\_size$, Goal piece type $goal\_piece\_type$, Number of comb free spaces $comb\_free\_spaces$
\Ensure State type $state\_type$
\If{$num\_of\_vacants \geq goal\_piece\_size$}
    \State $state\_type \gets \text{A}$
\Else
    \If{$goal\_piece\_type = \text{``I-shaped''}$}
        \If{$goal\_piece\_size \geq 3$}
            \State $state\_type \gets \text{B}$
        \Else \Comment{$n = 2$; since $k < n = 2$ and $k \ge 1$, we have $k = 1$}
            \State $state\_type \gets \text{C}$
        \EndIf
    \ElsIf{$goal\_piece\_type = \text{``L-shaped''}$}
        \If{$num\_of\_vacants = 1$}
            \State $state\_type \gets \text{D}$
        \Else
            \State $state\_type \gets \text{E}$
        \EndIf
    \ElsIf{$goal\_piece\_type = \text{``comb-shaped''} \land comb\_free\_spaces > num\_of\_vacants$}
        \State $state\_type \gets \text{F}$
    \Else
        \State \Comment{Reached by two independent cases: comb-shaped with $s \le k$, and ``other'' linearly symmetric shapes}
        \State $state\_type \gets \text{G}$
    \EndIf
\EndIf
\end{algorithmic}
\end{algorithm}

Algorithm~\ref{alg:state-type-assignment} runs in $O(1)$ via a fixed sequence of comparisons on initial-state parameters.

\subsection{Class A}
\label{subsec:class-a}

When $k \ge n$, the goal piece can execute jumping moves (Theorem~\ref{thm:general_move}), so the general heuristic of Section~\ref{sec:heuristics} provides accurate guidance. We apply $\hgeneral$ directly for Class~A.

\subsection{Class B}
\label{subsec:class-b}

Class~B arises when the goal piece is I-shaped with $n \ge 3$ and $k < n$. By Corollary~\ref{cor:crossing}, the piece cannot cross columns; all blocking pieces on the same column must be cleared. Let $d(s)$ be the vertical distance, $k(s)$ the vacancy count, $P_B(s)$ the blocking pieces on the goal path, and $g_p(s)$ the cells of $p$ on the path. The Class~B heuristic is
\[
	h_B(s) = \left\lceil \frac{d(s)}{k(s)} \right\rceil + \sum_{p \in P_B(s)} \left\lceil \frac{g_p(s)}{k(s)} \right\rceil .
\]

\begin{theorem}[Consistency of $h_B$]\label{thm:hB_consistency}
	$h_B$ is consistent: in every case $h_B(s) - h_B(s') \le 1$ (proof in Appendix~\ref{app:class-b-proof}). The same case-based methodology (advance goal piece / clear blocker / other move) serves as template for Classes~C--E.
\end{theorem}

\subsection{Class C}
\label{subsec:class-c}

Class~C arises when the goal piece is I-shaped with $n=2$ and $k=1$. Unlike Class~B, the piece can rotate and cross columns, invalidating the single-column premise.

Let $\text{top\_piece}(s)$, $\text{bot\_piece}(s)$, $\text{top\_goal}(s)$,
and $\text{bot\_goal}(s)$ denote the extremal row indices of the goal piece
and goal position.  The piecewise vertical distance is
\[
  d_v(s)=
  \begin{cases}
    |\text{top\_piece}(s)-\text{top\_goal}(s)|, & \text{goal above},\\
    |\text{bot\_piece}(s)-\text{bot\_goal}(s)|, & \text{goal below},\\
    0,                                           & \text{otherwise}.
  \end{cases}
\]
The orientation penalty $R(s)=1$ if the goal position is horizontal and the
goal piece must perform a terminal rotation after vertical alignment; otherwise
$R(s)=0$ (see Appendix~\ref{app:sn5-class-c-indicators}).  The frame-side discrepancy $F(s)=1$ if the goal piece and its
vertical goal position are on opposite frame sides; otherwise $F(s)=0$.
Set $d(s)=d_v(s)+R(s)+F(s)$.  The advance-ready indicator $M(s)=1$ if the
goal piece can advance without an extra clearing step (the front-facing unit in
the movement direction is already vacant); otherwise $M(s)=0$.

With $k=1$, the goal piece advances at most one unit per move, and a clearing move is needed before each advance, giving a lower bound of $2d$ moves. The Class~C heuristic is
\(
	h_C(s) = 2\bigl(d_v(s) + R(s) + F(s)\bigr) - M(s).
\)

\begin{theorem}[Consistency of $h_C$]\label{thm:hC_consistency}
	$h_C$ is consistent: in every case $h_C(s) - h_C(s') \le 1$ (proof in Appendix~\ref{app:class-c-proof}).
\end{theorem}

\subsection{Class D}
\label{subsec:class-d}

Class~D arises when the goal piece is L-shaped ($n=3$) and $k=1$. All blocking pieces on the path must be fully cleared.

Let $d_v(s) = |r^{\text{piece}}_H(s) - r^{\text{goal}}_H(s)|$ be the vertical offset of the horizontal arm:
\[
	h_D(s) = d_v(s) + \sum_{p \in P_D(s)} g_p(s) \quad \text{(since $k=1$)}.
\]

\begin{theorem}[Consistency of $h_D$]\label{thm:hD_consistency}
	$h_D$ is consistent: the argument follows Class~B with $k=1$ and row distance $d_v$ (proof in Appendix~\ref{app:class-d-proof}).
\end{theorem}

\subsection{Class E}
\label{subsec:class-e}

Class~E arises when the goal piece is L-shaped ($n=3$) and $k=2$. The key structural difference from Class~D is the bypass move: the goal piece can slide past a size-1 blocker without removing it (Figure~\ref{fig:class-e-bypass-slide}). Let $\delta(s) = 1$ if the bypass precondition is active, $0$ otherwise:
\[
	h_E(s) = d_v(s) + \sum_{p \in P_E(s)} \left\lceil \tfrac{g_p(s)}{2} \right\rceil - \delta(s).
\]

\begin{theorem}[Consistency of $h_E$]\label{thm:hE_consistency}
	$h_E$ is consistent. The genuinely new case is the bypass move: $d_v$ decreases by~2 while $\delta$ toggles from~1 to~0, giving $h_E(s) - h_E(s') = 2 - 1 = 1$ (proof in Appendix~\ref{app:class-e-proof}).
\end{theorem}

\begin{table}[htbp]
	\centering
	\caption{Closed-form heuristic definitions for Classes B--E. All heuristics assume unit
		step costs. $k(s)$ is the number of vacant units; $d(s)$ or $d_v(s)$ is the relevant
		vertical distance; $P_X(s)$ (for $X \in \{B, D, E\}$) is the set of blocking pieces on the goal path for class $X$;
		$g_p(s)$ is the number of path-cells occupied by piece $p$. Class C indicators
		$R(s)$, $F(s)$, $M(s)$ are defined in Section~\ref{subsec:class-c}; the Class E
		indicator $\delta(s)$ is defined in Section~\ref{subsec:class-e}.
		Class C does not use an explicit blocking-piece set $P_C(s)$; obstruction costs for this class are subsumed in the orientation penalty $R(s)$ and the frame-side discrepancy $F(s)$ (Section~\ref{subsec:class-c}).}
	\label{tab:heuristic-summary}
	\renewcommand{\arraystretch}{2.0}
	\begin{tabular}{@{}llll@{}}
		\toprule
		\textbf{Class} & \textbf{Conditions}                                          & \textbf{Heuristic formula} & \textbf{Theorem} \\
		\midrule
		B              & I-shaped, $n \ge 3$, $k < n$
		               & $h_B(s)=\displaystyle\Bigl\lceil\frac{d(s)}{k(s)}\Bigr\rceil
			+\sum_{p\in P_B(s)}\Bigl\lceil\frac{g_p(s)}{k(s)}\Bigr\rceil$
		               & \ref{thm:hB_consistency}                                                                                     \\
		C              & I-shaped, $n=2$, $k=1$
		               & $h_C(s)=2\bigl(d_v(s)+R(s)+F(s)\bigr)-M(s)$
		               & \ref{thm:hC_consistency}                                                                                     \\
		D              & L-shaped, $n=3$, $k=1$
		               & $h_D(s)=d_v(s)+\displaystyle\sum_{p\in P_D(s)}g_p(s)$
		               & \ref{thm:hD_consistency}                                                                                     \\
		E              & L-shaped, $n=3$, $k=2$
		               & $h_E(s)=d_v(s)+\displaystyle\sum_{p\in P_E(s)}
			\Bigl\lceil\tfrac{g_p(s)}{2}\Bigr\rceil-\delta(s)$
		               & \ref{thm:hE_consistency}                                                                                     \\
		\bottomrule
	\end{tabular}
\end{table}

\subsection{Class F}\label{subsec:class-f}

Class~F arises when the goal piece is comb-shaped and $s > k$. By Theorem~\ref{thm:comb_flip}, the comb's Body cannot flip to the opposite column, so all blocking pieces on the body side must be cleared:
\[
	h_F(s) = \left\lceil \frac{d}{k} \right\rceil + \sum_{p_i \in \text{Blocking pieces}} \left\lceil \frac{g_i}{k} \right\rceil.
\]
Consistency follows identically to Class~B (Appendix~\ref{app:class-b-proof}).

\subsection{Class G}
\label{subsec:class-g}

Class~G collects the residual cases: comb-shaped pieces with $s \le k$, and ``other'' linearly symmetric shapes. Due to the high variety of shapes, the heuristic uses only the vertical distance:
\(
	h_G = \lceil d/k \rceil.
\)
Consistency follows from the Class~B argument restricted to the distance term. Table~\ref{tab:heuristic-summary} summarizes the closed-form heuristic expressions for Classes~B--E; Classes~F and~G are defined in Sections~\ref{subsec:class-f} and~\ref{subsec:class-g}.

\subsection{Class-Based Tie-Breaker Selection}\label{subsec:tiebreaker-selection}

Depth tie-breaking is used for Classes~A--E; vertical distance tie-breaking is used for Classes~F and~G. In F and G, $\lceil d/k \rceil$ remains unchanged for up to $k-1$ successive vertical advances, so depth no longer correlates with progress; the vertical distance $d$ is a more direct guide. To our knowledge, class-conditional tie-breaker switching has not been explored previously; it is the key novelty of our tie-breaking contribution \cite{asai2016tiebreaking, correa2018analyzing}.

\section{Results and Discussion}\label{sec:results}

We compare BFS (baseline), Standard A* (SA*), Depth-Prioritized A* (DPA*), and Class-Based Heuristic A* (CBHA*) across seven puzzle classes.

\subsection{Results}

CBHA* solves 93.4\% of cases versus 64\% (DPA*), 39\% (SA*), and 17\% (BFS) (Figure~\ref{fig:merged-overall}(a)). Performance gaps are conservative: memory-exhausted runs (which disproportionately affected BFS and Standard A*) were excluded from averages. CBHA* achieves an average EBF of 3---roughly $7\times$ below SA* and $26\times$ below BFS---by pruning the search tree to kinematically relevant trajectories via class-specific heuristics. This reduces node expansions by 87.98\% compared to Standard~A* and 91.31\% compared to BFS, effectively curtailing the exponential search-space growth characteristic of NP-complete spatial reasoning. The depth analysis (Figure~\ref{fig:merged-overall}(d)) shows CBHA* reaches average depths of 12, versus 4--7 for the baselines, confirming focused rather than broad exploration.

\begin{figure}[!htbp]
	\centering
	\includegraphics[width=\textwidth]{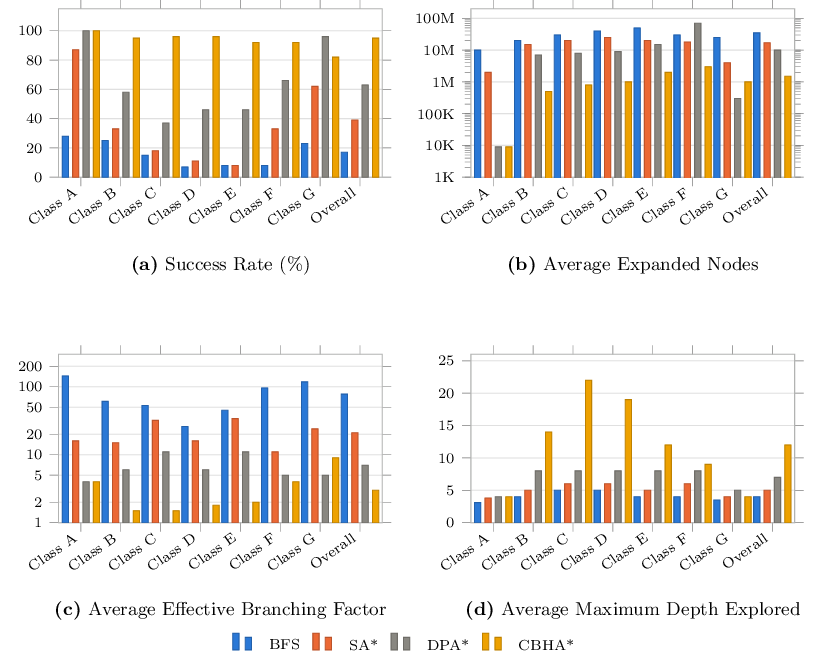}
	\caption{Performance comparison of BFS, SA*, DPA*, and CBHA* across the seven
kinematic classes (A--G) and overall. Panel~(a): success rate (\%); (b):
average expanded nodes (log scale); (c): average effective branching factor
(log scale); (d): average maximum depth explored.}
	\label{fig:merged-overall}
\end{figure}

\begin{table}[htbp]
	\centering
	\caption{Full experimental results by algorithm and class: number of
		instances, success rate, average expanded nodes, effective branching
		factor (EBF), and average maximum solution depth. Overall success rates
		cited in the text are weighted averages over all 146 instances.}
	\label{tab:per-class-full-results}
	\small
	\begin{tabular}{llrrrrr}
		\toprule
		Class & Algorithm & $N$ & Success (\%) & Avg.\ nodes          & EBF & Avg.\ max depth \\
		\midrule
		A     & BFS       & \multirow{4}{*}{32} & 28           & $\geq$8{,}955{,}252  & 144 & 3               \\
		A     & SA*       &                     & 87           & 1{,}626{,}488        & 16  & 3.8             \\
		A     & DPA*      &                     & 100          & 8{,}941              & 4   & 4               \\
		A     & CBHA*     &                     & 100          & 8{,}941              & 4   & 4               \\
		\midrule
		B     & BFS       & \multirow{4}{*}{24} & 25           & $\geq$13{,}681{,}738 & 61  & 4               \\
		B     & SA*       &                     & 33           & $\geq$10{,}765{,}203 & 15  & 5               \\
		B     & DPA*      &                     & 58           & 3{,}635{,}561        & 6   & 8               \\
		B     & CBHA*     &                     & 95           & 572{,}268            & 1.5 & 14              \\
		\midrule
		C     & BFS       & \multirow{4}{*}{27} & 15           & $\geq$17{,}224{,}895 & 53  & 5               \\
		C     & SA*       &                     & 18           & $\geq$14{,}668{,}560 & 32  & 6               \\
		C     & DPA*      &                     & 37           & 6{,}205{,}728        & 11  & 8               \\
		C     & CBHA*     &                     & 96           & 657{,}010            & 1.5 & 22              \\
		\midrule
		D     & BFS       & \multirow{4}{*}{26} & 7            & 18{,}861{,}834       & 26  & 5               \\
		D     & SA*       &                     & 11           & 16{,}445{,}783       & 16  & 6               \\
		D     & DPA*      &                     & 46           & 5{,}403{,}719        & 6   & 8               \\
		D     & CBHA*     &                     & 96           & 564{,}743            & 1.8 & 19              \\
		\midrule
		E     & BFS       & \multirow{4}{*}{13} & 8            & 20{,}292{,}069       & 45  & 4               \\
		E     & SA*       &                     & 8            & 17{,}994{,}491       & 34  & 5               \\
		E     & DPA*      &                     & 46           & 6{,}535{,}319        & 11  & 8               \\
		E     & CBHA*     &                     & 92           & 1{,}565{,}849        & 2   & 12              \\
		\midrule
		F     & BFS       & \multirow{4}{*}{12} & 8            & 16{,}906{,}471       & 96  & 4               \\
		F     & SA*       &                     & 33           & 11{,}833{,}676       & 11  & 6               \\
		F     & DPA*      &                     & 66           & 38{,}448{,}871       & 5   & 8               \\
		F     & CBHA*     &                     & 92           & 2{,}663{,}743        & 4   & 9               \\
		\midrule
		G     & BFS       & \multirow{4}{*}{12} & 23           & 14{,}598{,}308       & 118 & 3.5             \\
		G     & SA*       &                     & 62           & 6{,}624{,}414        & 24  & 4               \\
		G     & DPA*      &                     & 96           & 32{,}816             & 5   & 5               \\
		G     & CBHA*     &                     & 82           & 3{,}575{,}216        & 9   & 5               \\
		\bottomrule
	\end{tabular}
\end{table}

Table~\ref{tab:per-class-full-results} reports the full per-class breakdown across all 146 instances (Classes A--G). Overall success rates cited in the text are instance-weighted averages: e.g., DPA* achieves 64\% overall ($\frac{32{\times}100+24{\times}58+27{\times}37+26{\times}46+13{\times}46+12{\times}66+12{\times}96}{146}\approx63.9\%$), not the unweighted class mean of 58\%. In Class~A, the large vacancy-to-size ratio keeps solution depth shallow, making all algorithms relatively effective. Class~C has the deepest average solutions (${\sim}22$ for CBHA*) due to the single-vacancy alternation constraint that doubles solution length. Classes~D and~E illustrate the bypass-move distinction: the only structural difference ($k{=}1$ vs.\ $k{=}2$) explains why $h_D$ is inadmissible on Class~E states and why CBHA* solves 92\% of Class~E versus 96\% of Class~D. Class~F is dominated by the flip-impossibility (Theorem~\ref{thm:comb_flip}), requiring multi-step coordination per unit of vertical progress. Class~G is the sole class where DPA* outperforms CBHA*: at small vertical distances $h_G \approx 0$, making the class-specific heuristic nearly uninformative.

Section~\ref{sec:hard-configurations} identifies configurations where large blocking pieces must be displaced in coordination with the target piece, creating solution depths beyond the reach of all algorithms under the resource budget.

\subsubsection{Challenging Configurations Beyond the Reach of Our Algorithms}\label{sec:hard-configurations}

These configurations occur in Classes~B--G, where blocking pieces of size exceeding $k$ are positioned in the target piece's path, requiring coordinated sliding sequences. Illustrative examples are shown in Figure~\ref{fig:hard-configurations}.

\subsection{Discussion}

As observed, the standard $A^*$ algorithm (SA*) implemented with our general heuristic performs significantly better than BFS in Class~A examples. However, in Classes~B--G, this heuristic fails to accurately represent the distance to the goal. Because the goal positions often remain occupied until the final moves, the heuristic's reliance on goal-occupancy metrics causes the algorithm to operate almost blindly in constrained spaces. DPA* outperforms SA* by prioritizing deeper nodes in $f$-value plateaus. This approach works effectively because deeper nodes that retain favorable heuristic values are often closer to the solution. Despite this improvement, DPA* still fails on Classes~B--G with larger frames, frequently exceeding time or memory limits, motivating the need for class-specific heuristics.

Three principles emerge from our findings: (1)~tie-breaking materially affects efficiency when heuristic range is small and the search space is massive; (2)~a kinematic taxonomy enables class-specific admissible heuristics that significantly outperform $\hgeneral$~\cite{russell2020artificial}; and (3)~the optimal tie-breaking rule is class-dependent rather than universal. Together, these constitute an instantiation of the algorithm-selection paradigm at the heuristic level.

By dynamically adapting to the structure of each puzzle class, Class-Based Heuristic Selection (CBHS) delivers highly accurate estimates. Furthermore, class-conditional tie-breaker selection proved essential: while the depth tie-breaker (prioritizing deeper nodes) produced the best results in Classes~A--E, vertical distance tie-breaking proved more effective in Classes~F and~G.

In Class~G, $h_G$ underperforms at small vertical distances, assigning near-zero values across a wide range of states (which allowed DPA* to temporarily outperform CBHA* in these specific shallow-depth cases). While $h_G$ theoretically improves as vertical distance and solution depth increase, refining Class~G into subclasses with tighter heuristics, or formulating a strictly consistent piecewise heuristic specifically tailored to capture constraints in small-$d$ regimes, are immediate extensions.

From another perspective, SA* and DPA* represent more general-purpose approaches. Importantly, the general heuristic can be used for all versions of the flying block puzzle. In contrast, the heuristics developed for CBHA* are specifically tailored for the NP-complete subclass analyzed in this study, highlighting a fundamental trade-off: CBHA* achieves its superior performance by sacrificing algorithmic generality for deep, domain-specific knowledge.

\textbf{Limitations.} Despite CBHA*'s strong overall performance, extreme structural complexity remains a challenge. In configurations where large blocking pieces must be moved in coordination with the target piece, exceptionally deep and highly interdependent solution paths are created that expose the limits of heuristic-guided search. Additionally, our 12~GB RAM constraint artificially diminished the observed performance differences between algorithms; without it, the true disparity in node expansions and efficiency between CBHA*, DPA*, and the baselines would be considerably more pronounced.

Furthermore, the CBHA* framework requires domain-specific expert effort to derive the kinematic taxonomy and class heuristics. Developing a novel, provably consistent heuristic is a knowledge-intensive task requiring rigorous mathematical proofs. The framework's performance gains are tied to the structural regularity of the 2-column variant; extensions to wider frames or irregular domains would require re-deriving the class partition. Automating this taxonomy derivation and heuristic design (e.g., via machine learning, pattern recognition, or meta-learning) is an open challenge and a longer-term direction for generalizing CBHS concepts to novel puzzle families.

\textbf{Generalization and Domain-Transfer of CBHS.} Theorem~\ref{thm:general_move}'s overlap bound is domain-agnostic: whenever $k < n$, any valid transition preserves $n-k$ cells from the previous position. In Sokoban~\cite{junghanns2001sokoban}, narrow-corridor LA-MAPF~\cite{li2019multi}, and BRP~\cite{caserta2012mathematical}, analogous clearance-to-size constraints produce the same plateau structure, making class-triggered heuristic switching and tie-breaker selection directly relevant. The kinematic taxonomy finds structural parallels in industrial planning constraints; a detailed correspondence is provided in Appendix~\ref{app:sn8-industrial-mapping}. Appendix~\ref{app:sn8-industrial-mapping} describes a representative puzzle instance for each class alongside its physical analogue, concretising how the kinematic constraints manifest in both domains. Appendices~\ref{sn10-1-multi-class-transition-narrative-pipelines} and~\ref{sn10-2-detailed-operational-transitions} extend this mapping with integrated multi-class transition narrative pipelines---including an automated warehousing logistics fleet and an articulated turbine inspection arm---in which a single continuous mission traverses Classes~A, B, C, E, and~G, or Classes~A, D, E, and~F sequentially. These transitions make explicit that CBHS's runtime class detection is not merely a static classification tool: as an active agent moves from an unconstrained dispatch zone (Class~A) through narrow aisles (Class~B) and tight turning junctions (Class~C), it must dynamically adapt its heuristic and tie-breaking strategy on-the-fly, as required by the kinematic constraints of each class established in Theorem~\ref{thm:partition}. Deriving admissible class-specific heuristics for these domains requires adapting the geometry-specific terms ($d_v$, $R$, $F$, $\delta$) to each domain's piece kinematics.

\section{Methods}\label{sec:methods}

\subsection{Puzzle Instance Generation}

A total of 146 test instances were constructed across the seven kinematic classes
(A--G) defined in Section~\ref{sec5}, designed to span a wide range of
structural configurations and solution depths. Instance diversity was ensured
along three axes: piece positions, piece types and sizes, and the number of
vacant units; pseudo-random placement was used throughout subject to the
solvability constraints established in Section~\ref{sec:formulation}. For each
class, the goal-piece type and size were fixed to satisfy the class
definition (Definitions~\ref{def:i-shaped}--\ref{def:other-type} and
Algorithm~\ref{alg:state-type-assignment}), while blocking pieces were placed
randomly within feasible frame configurations.

Within each class, instances were ordered by increasing difficulty. In
Class~A, difficulty increases as the difference $k - n$ between the number of
vacant units $k$ and the goal-piece size $n$ decreases; in Classes~B--G,
difficulty is primarily controlled by the vertical offset between the goal
piece's initial position and its target location---larger offsets require longer
sliding sequences and more complex coordination, resulting in deeper optimal
solution paths. Two Class~A outlier instances (frame heights 200 and 500) were
excluded from aggregate statistics because their atypically large legal-move
sets distorted EBF estimates; failure counts still include these cases.

\subsection{Algorithm Implementation}

All four algorithms---BFS, Standard A* (SA*), Depth-Prioritized A* (DPA*), and
Class-Based Heuristic A* (CBHA*)---were implemented in Python and share a common
state-representation and move-generation layer, ensuring a fair basis for
comparison. The open list in SA*, DPA*, and CBHA* is maintained as a min-heap.
SA* orders by $f$ alone. DPA* orders lexicographically by $(f, -g)$ to
implement depth-based tie-breaking (Section~\ref{sec4}). CBHA* uses the same
$(f, -g)$ ordering for Classes~A--E; for Classes~F and~G it orders by
$(f, d)$, preferring smaller vertical distance~$d$ within $f$-value ties
(Section~\ref{subsec:tiebreaker-selection}). BFS uses a standard FIFO queue.

Because every class heuristic developed in this work is consistent
(Theorems~\ref{thm:general_consistency}, \ref{thm:hB_consistency}--\ref{thm:hE_consistency},
and Sections~\ref{subsec:class-f}--\ref{subsec:class-g}), graph-search mode with
duplicate detection via a closed set is used for all classes; consistent
heuristics guarantee that a state first expanded via graph search carries the
optimal $g$-value, so the closed set can be maintained without loss of
optimality. The class of a puzzle instance is determined at initialisation by
Algorithm~\ref{alg:state-type-assignment} in $O(1)$ time; the corresponding
heuristic function and tie-breaking rule are then fixed for the entire search.
Three candidate tie-breaking criteria were evaluated prior to selecting
the class-conditional rules above: (i)~number of misplaced units,
(ii)~sum of piece-to-goal distances, and (iii)~depth in the search tree
($g$-value). Depth prioritization yielded the best empirical performance in
Classes~A--E, while vertical-distance ordering was preferred for Classes~F
and~G (Section~\ref{subsec:tiebreaker-selection}), consistent with the
observation that deeper nodes with equal $f$-values are more likely to lie on a
path leading to the goal in the former regime (Section~\ref{sec4}). Source code and datasets are
available at \url{https://github.com/Pedyi/Flying-Block-Puzzle}.

\subsection{Evaluation Protocol}

Each instance was solved independently by all four algorithms under a resource
budget of 12~GB RAM and 30 minutes of wall-clock time. An instance was recorded
as a failure if either limit was reached before a solution was found; the number
of expanded nodes and the EBF at the moment the limit was hit were recorded
separately and used when computing aggregate failure statistics, but excluded
from per-algorithm averages of node expansions and EBF (which cover successful
runs only). The following metrics were recorded for each run:

\begin{itemize}
    \item \textbf{Success rate}: percentage of instances solved within the
    resource budget.
    \item \textbf{Average node expansions}: mean number of states expanded
    before reaching the goal, averaged over successful runs only.
    \item \textbf{Effective branching factor (EBF)}: the branching factor $b^*$
    satisfying $N + 1 = b^* + b^{*2} + \cdots + b^{*d}$, where $N$ is the
    number of nodes expanded and $d$ is the solution depth, solved numerically;
    lower EBF values indicate more focused search.
    \item \textbf{Average maximum depth}: mean depth of the deepest node
    expanded during search, over all runs (successful and failed); this
    reflects the breadth of state-space exploration rather than solution quality,
    since all algorithms are guaranteed to return optimal solutions on success.
\end{itemize}

Per-class averages exclude Class~A from aggregate comparisons owing to the
shallow, single-step solution structure of Class~A instances. Failure counts
(memory or time limit exceeded) are reported separately in
Table~\ref{tab:per-class-full-results}.

Aggregate figures quoted in the abstract and Section~\ref{sec:results} --- including average EBF, node-expansion reduction percentages, and overall success rate --- are instance-weighted means over all runs in Classes~B--G; Class~A is excluded as described in Section~\ref{sec:methods}, and the two Class~A outlier instances (frame heights 200 and 500) excluded from aggregate statistics per Section~\ref{sec:methods} are likewise excluded here. Success rate is computed as the fraction of instances solved within the resource budget summed over all non-Class-A classes, weighted by the number of instances per class.


\section*{Acknowledgements}
Not applicable.

\section*{Declarations}

\begin{itemize}
\item \textbf{Author contributions}: Sanyar Ahmadi and Pedram Asadzadeh
contributed equally to the algorithm design, analysis and implementation under
supervision of Amanj Khorramian.
\item \textbf{Funding}: Not applicable.
\item \textbf{Conflict of interest}: S.A., P.A., and A.K. declare no competing interests.
\item \textbf{Ethics approval}: Not applicable.
\item \textbf{Code availability}: See the Data Availability statement below.
\end{itemize}

\section*{Data Availability}
The datasets and source code used in this study are publicly available at \url{https://github.com/Pedyi/Flying-Block-Puzzle}.



\appendix

\renewcommand{\thefigure}{\thesection\arabic{figure}}
\renewcommand{\thetable}{\thesection\arabic{table}}
\renewcommand{\thealgorithm}{\thesection\arabic{algorithm}}
\setcounter{figure}{0}
\setcounter{table}{0}
\setcounter{algorithm}{0}

\section{Proofs of Theorem~\ref{thm:general_move} (General Move Constraint), Corollary~\ref{cor:jumping} (Jumping Move Constraint), and Corollary~\ref{cor:crossing} (Crossing Constraint)}\label{app:sn1-proofs}

\begin{theorem}[General Move Constraint]
	Let $P$ be a piece of size $n$. Let its position before a move be the set of cells $C_P$ and its position after the move be $C'_P$, where $|C_P| = |C'_P| = n$. If the total number of vacant cells available in the puzzle frame is $k$, then any valid move must satisfy the condition
	$|C_P \cap C'_P| \ge n - k$.
\end{theorem}

\begin{proof}
	In any valid move, every newly occupied cell must have been vacant before the
	move, so $C'_P \setminus C_P \subseteq V$, where $V$ is the set of vacant
	cells with $|V| = k$. Hence $|C'_P \setminus C_P| \le k$. Since
	$|C'_P| = n$ and $C'_P$ is the disjoint union of $C_P \cap C'_P$ and
	$C'_P \setminus C_P$, we obtain
	\begin{equation*}
		|C_P \cap C'_P| = n - |C'_P \setminus C_P| \ge n - k. \qedhere
	\end{equation*}
\end{proof}

\begin{corollary}[Jumping Move Constraint]
	Let $P$ be a piece of size $n$, and let $k$ be the number of vacant units in
	the frame. If $k < n$, then $P$ cannot execute a jumping move.
\end{corollary}

\begin{proof}
	A jumping move requires $C_P \cap C'_P = \emptyset$, whereas
	the General Move Constraint requires
	$|C_P \cap C'_P| \ge n - k \ge 1$ when $k < n$.
\end{proof}

\begin{corollary}[Crossing Constraint]
	Let $P$ be a rectangular piece of size $n > 2$ within a $2 \times h$ puzzle
	frame $F$, and let $k < n$ be the total number of vacant units. Suppose $P$
	initially occupies a set of cells $C_P \subset F$ contained entirely in one
	column of the frame, and let $C_P^{\text{goal}} \subset F$ be a target
	configuration contained entirely in the other column, so that
	$C_P \cap C_P^{\text{goal}} = \emptyset$. Then no sequence of valid moves
	transforms $C_P$ into $C_P^{\text{goal}}$.
\end{corollary}

\begin{proof}
	Since $P$ is rectangular, occupies $n > 2$ cells, and fits inside a single
	column of a frame of width $2$, $P$ is a $1 \times n$ piece. A horizontal
	placement of a $1 \times n$ piece with $n > 2$ does not fit in a frame of
	width $2$, so every valid placement of $P$ in $F$ is a vertical segment
	contained entirely in column $1$ or entirely in column $2$.

	Consider any sequence of valid placements
	$C_P^{(0)} = C_P, C_P^{(1)}, \dots, C_P^{(T)} = C_P^{\text{goal}}$, and
	assign to each placement the index of the column containing it. The first
	and last placements lie in different columns, so there exists a step $t$ at
	which the column index changes. The two columns are disjoint sets of cells,
	hence $C_P^{(t)} \cap C_P^{(t+1)} = \emptyset$, i.e.\ step $t$ is a jumping
	move. By the Jumping Move Constraint, this is impossible when $k < n$.
\end{proof}

\section{Proof of Theorem~\ref{thm:comb_flip} (Comb-Piece Flip Constraint)}\label{app:sn2-comb-flip}

\begin{theorem}[Comb-Piece Flip Constraint]
	Let $P$ be a comb-shaped piece with occupied cells $C_P=\text{Body}\cup\text{Teeth}$ and $s:=|\text{Free}|$. Let $V\subset F$ denote the vacant cells in the frame and $k:=|V|$. If $k<s$, then no sequence of valid moves transforms $C_P$ into a configuration $C'_P$ obtained by reflecting $P$ across its horizontal axis of symmetry.
\end{theorem}

\begin{proof}
	We use proof by contradiction. Assume that a vertical flip from $C_P$ to $C'_P$ is possible despite $k<s$. Let $\sigma$ denote reflection across the comb's horizontal axis of symmetry, so $C'_P=\sigma(C_P)$.

	For this move to be valid, the set of newly occupied cells, $C'_P\setminus C_P$, must lie in the available vacant cells $V$. Hence
	$|C'_P\setminus C_P|\le k$.
	This also follows from the General Move Constraint: $|C'_P\setminus C_P|=n-|C_P\cap C'_P|$ and the constraint gives $n-k\le |C_P\cap C'_P|$, which rearranges to $|C'_P\setminus C_P|\le k$ \quad (a).

	By the comb symmetry, the reflected body satisfies $\text{Free}\subseteq\sigma(\text{Body})\subseteq C'_P$. Since $\text{Free}\cap C_P=\emptyset$ (free cells are not occupied by the piece) and $\text{Free}\subseteq C'_P$, we have $\text{Free}\subseteq C'_P\setminus C_P$. Therefore
	$s=|\text{Free}|\le |C'_P\setminus C_P|$ \quad (b).
	Combining (a) and (b) yields $s\le k$, contradicting $k<s$. Hence no such flip is possible.
\end{proof}

\section{Primitive Heuristic --- Definition, Algorithm, and Consistency Proof}\label{app:sn3-primitive}

\begin{definition}[Primitive Heuristic Function]
  \label{def:primitive_heuristic}
  Let $G_{pos}$ be the goal cells and $P(s)$ the set of non-target pieces
  occupying at least one cell in $G_{pos}$.  The primitive heuristic is
  $h_p(s)=|P(s)|+\delta(s)$, where $\delta(s)=1$ if $s$ is not a goal state
  and $0$ otherwise.
\end{definition}

\begin{algorithm}[H]
\caption{Heuristic Calculation for the Primitive Heuristic}
\label{alg:primitive-heuristic}
\begin{algorithmic}[1]
\Require state $s$, goal $g$
\Ensure heuristic\_value
\State goal\_occupied\_positions $\leftarrow$ \{ $(i, j) \mid goal[i][j] \neq 0$ \}
\State Unique\_pieces $\leftarrow$ empty set
\ForAll{$(i, j) \in$ goal\_occupied\_positions}
    \If{board[$i$][$j$] $\neq 0$ \textbf{and} board[$i$][$j$] $\neq$ goal\_piece}
        \State Unique\_pieces.add(board[$i$][$j$])
    \EndIf
\EndFor
\State heuristic\_value $\leftarrow$ size of Unique\_pieces
\If{\textbf{not} is\_goal(board, goal)}
    \State heuristic\_value $\leftarrow$ heuristic\_value + 1
\EndIf
\State \Return heuristic\_value
\end{algorithmic}
\end{algorithm}

\begin{remark}
Algorithm~\ref{alg:primitive-heuristic} (primitive heuristic) and Algorithm~\ref{alg:general-heuristic}
(general heuristic) share the same outer structure --- iterating over goal
positions and collecting a blocking set --- but differ materially in how
the accumulated value is computed.  Algorithm~\ref{alg:primitive-heuristic} builds a set of piece
identifiers and returns its cardinality, yielding $|P(s)|$.  Algorithm~\ref{alg:general-heuristic}
iterates over that same set and accumulates the
weighted sum $\sum_{p_i} \lceil g_i / k \rceil$.  The structural
similarity is intentional: the general heuristic is a strict refinement of
the primitive heuristic, replacing the unit contribution per piece with a
move-count lower bound derived from the General Move Constraint
(Theorem~\ref{thm:general_move}).
\end{remark}

\begin{theorem}[Consistency of $h_p(s)$]
	\label{thm:consistency}
	The primitive heuristic function $h_p(s) = |P(s)| + \delta(s)$ is consistent.
\end{theorem}

\begin{proof}
	We verify the two conditions for consistency: $h(s_{\text{goal}}) = 0$, and $h(s) \le c(s, a, s') + h(s')$ for every
	state--action--successor triple $(s, a, s')$, where $c(s, a, s') = 1$ for unit step costs.

	\textbf{Condition 1.} In a goal state, the target piece occupies all of
	$G_{pos}$, so no non-target piece occupies a cell of $G_{pos}$ and
	$P(s_{\text{goal}}) = \emptyset$; moreover $\delta(s_{\text{goal}}) = 0$.
	Hence $h_p(s_{\text{goal}}) = 0$.

	\textbf{Condition 2.} If $s$ is a goal state, then $h_p(s) = 0$ and the
	inequality is immediate; assume therefore $\delta(s) = 1$. A single move
	displaces exactly one piece, so $|P(s)|$ changes by at most one.

	\emph{Case A: a non-target piece moves and $|P|$ does not decrease.} Then $|P(s')| \ge |P(s)|$. A move of a
	non-target piece cannot create a goal state, so $\delta(s') = 1$ and
	$h_p(s) - h_p(s') \le 0$.

	\emph{Case B: a non-target piece moves entirely out of $G_{pos}$.} Then
	$|P(s')| = |P(s)| - 1$, and $s'$ is not a goal state, so $\delta(s') = 1$ and $h_p(s) - h_p(s') = 1$.

	\emph{Case C: the target piece moves.} Moving the target does not change
	which non-target pieces occupy $G_{pos}$, so $|P(s')| = |P(s)|$. If $s'$ is
	a goal state, then in $s$ every cell of $G_{pos}$ was vacant or occupied by
	the target, so $|P(s)| = 0$; hence $h_p(s) = 1$, $h_p(s') = 0$, and
	$h_p(s) - h_p(s') = 1$. Otherwise $\delta(s') = 1$ and
	$h_p(s) - h_p(s') = 0$.

	In every case $h_p(s) - h_p(s') \le 1$, so $h_p$ is consistent. This case-analysis structure is reused in all subsequent consistency proofs.
\end{proof}

\section{Proof of General Heuristic Consistency and Derivation}\label{app:sn4-general}

\subsection*{Derivation of the General Heuristic}

The primitive heuristic provides an effective estimate of the cost to solve the puzzle, especially when the pieces occupying the goal positions are small or can be cleared with minimal effort. However, in scenarios where the pieces occupying the goal positions are large and the number of vacant cells is limited, it may be impossible to clear an individual blocking piece from the goal region in a single move. As a result, the heuristic underestimates the actual cost and becomes less effective in guiding the search.

To improve the accuracy of our heuristic, we developed a more sophisticated general heuristic based on the observation that some pieces occupying the goal positions may require more than one move to be cleared. This improved heuristic is derived from the General Movement Constraint (Theorem~\ref{thm:general_move}). According to this rule, the number of moves required to clear a piece depends on the number of vacant units in the puzzle relative to the size of the piece.

By the General Move Constraint, for a piece of size $p$, if there are $k$ vacant units, the piece must overlap at least $p-k$ units of its previous position in every move. Each valid move displaces at most $k$ cells, while at least $p - k$ cells remain in place. This process repeats until the piece is fully cleared from the goal positions. Consequently, to fully clear a piece occupying $p$ goal positions using $k$ vacant units, the minimum number of moves required is $\left\lceil \frac{p}{k} \right\rceil$, where the ceiling function accounts for the fractional remainder after division.

\subsection*{Consistency Proof}

\begin{theorem}[Consistency of $\hgeneral(s)$]
	\emph{(Establishes Theorem~\ref{thm:general_consistency}.)}
	The general heuristic function $\hgeneral(s) = \sum_{p_i \in P(s)} \lceil g_i/k
	\rceil + \delta(s)$ is consistent: $h(s) \le c(s, a, s') + h(s')$ for all state--action--successor triples, with
	$c(s, a, s') = 1$.
\end{theorem}

\begin{proof}
	\textbf{Condition 1.} In a goal state, $P(s_{\text{goal}}) = \emptyset$ and
	$\delta(s_{\text{goal}}) = 0$, so $\hgeneral(s_{\text{goal}}) = 0$.

	\textbf{Condition 2.} If $s$ is a goal state the inequality is immediate;
	assume $\delta(s) = 1$. A single move displaces exactly one piece $p$.

	\emph{Case A: $p$ is a non-target piece.} Only the summand of $p$ can
	change. By the General Move Constraint, the number of cells vacated by
	$p$ in one move equals $|C'_p \setminus C_p| \le k$, so $g_p$ decreases by
	at most $k$ and $\lceil g_p / k \rceil$ decreases by at most $1$. A non-target move cannot create a goal state, so
	$\delta(s') = 1$ and $\hgeneral(s) - \hgeneral(s') \le 1$.

	\emph{Case B: $p$ is the target piece.} The summation is unchanged, since
	the occupancy of $G_{pos}$ by non-target pieces is unaffected. If $s'$ is a
	goal state, then in $s$ every cell of $G_{pos}$ was vacant or occupied by
	the target, so $P(s) = \emptyset$ and $\hgeneral(s) = 1$, while
	$\hgeneral(s') = 0$; the difference is $1$. Otherwise $\delta(s') = 1$ and
	the difference is $0$.

	In every case $\hgeneral(s) - \hgeneral(s') \le 1$, so $\hgeneral$ is
	consistent.
\end{proof}

\section{Class C Heuristic Indicator Functions}\label{app:sn5-class-c-indicators}

The indicator functions $R(s)$, $F(s)$, $M(s)$ and the piecewise vertical
distance $d_v(s)$ used in $h_C$ are defined in Section~\ref{subsec:class-c}
of the main manuscript.  This appendix is retained for cross-reference continuity.

\section{Heuristic Formula Summary for Classes B--G}\label{app:heuristic-algorithms}

Table~\ref{tab:heuristic-summary} in the main text collects the closed-form heuristic definitions for Classes~B--E; Classes~F and~G are defined in Sections~\ref{subsec:class-f} and~\ref{subsec:class-g}. This appendix provides the detailed notation reference: $k(s)$ is the number of vacant units; $d(s)$ or $d_v(s)$ is the relevant vertical distance; $P_X(s)$ is the set of blocking pieces on the goal path for class~$X$; $g_p(s)$ is the number of path-cells occupied by piece~$p$. Class~C indicators $R(s)$, $F(s)$, $M(s)$ are defined in Section~\ref{subsec:class-c}; the Class~E indicator $\delta(s)$ is defined in Section~\ref{subsec:class-e}.

\section{Examples of Classes A Through G}\label{app:class-examples}

To illustrate the implementation of our proposed heuristics, we present detailed worked examples for Classes A and C below. These two classes serve as representative samples that demonstrate the underlying mechanics and step-by-step execution of our approach. To maintain conciseness, the full walkthroughs for the remaining classes have been omitted, as they follow analogous analytical principles.

\subsection{Class A}

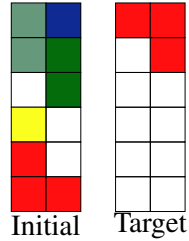
\begin{figure}[h]
	\centering
	\resizebox{0.15\textwidth}{!}{
		\begin{tikzpicture}

			\tikzstyle{every node}=[font=\Huge]
			\draw [ fill={rgb,255:red,102; green,153; blue,124} ] (-6.25,-5) rectangle (-5,-6.25);
			\draw [ fill={rgb,255:red,16; green,42; blue,147} ] (-3.75,-5) rectangle (-5,-6.25);
			\draw [ fill={rgb,255:red,102; green,153; blue,124} ] (-6.25,-6.25) rectangle (-5,-7.5);
			\draw [ fill={rgb,255:red,4; green,108; blue,11} ] (-3.75,-6.25) rectangle (-5,-7.5);
			\draw  (-6.25,-7.5) rectangle (-5,-8.75);
			\draw [ fill={rgb,255:red,4; green,108; blue,11} ] (-3.75,-7.5) rectangle (-5,-8.75);
			\draw [ fill={rgb,255:red,251; green,255; blue,31} ] (-6.25,-8.75) rectangle (-5,-10);
			\draw  (-3.75,-8.75) rectangle (-5,-10);
			\draw [ fill={rgb,255:red,255; green,15; blue,15} ] (-6.25,-10) rectangle (-5,-11.25);
			\draw  (-3.75,-10) rectangle (-5,-11.25);
			\draw [ fill={rgb,255:red,255; green,15; blue,15} ] (-6.25,-11.25) rectangle (-5,-12.5);
			\draw [ fill={rgb,255:red,255; green,15; blue,15} ] (-3.75,-11.25) rectangle (-5,-12.5);
			\draw [ fill={rgb,255:red,255; green,15; blue,15} ] (-2.5,-5) rectangle (-1.25,-6.25);
			\draw [ fill={rgb,255:red,255; green,15; blue,15} ] (0,-5) rectangle (-1.25,-6.25);
			\draw  (-2.5,-6.25) rectangle (-1.25,-7.5);
			\draw [ fill={rgb,255:red,255; green,15; blue,15} ] (0,-6.25) rectangle (-1.25,-7.5);
			\draw  (-2.5,-7.5) rectangle (-1.25,-8.75);
			\draw  (0,-7.5) rectangle (-1.25,-8.75);
			\draw  (-2.5,-8.75) rectangle (-1.25,-10);
			\draw  (0,-8.75) rectangle (-1.25,-10);
			\draw  (-2.5,-10) rectangle (-1.25,-11.25);
			\draw  (0,-10) rectangle (-1.25,-11.25);
			\draw  (-2.5,-11.25) rectangle (-1.25,-12.5);
			\draw  (0,-11.25) rectangle (-1.25,-12.5);
			\node [font=\Huge] at (-5,-13) {\scalebox{1.2}{Initial}};
			\node [font=\Huge] at (-1.25,-13) {\scalebox{1.2}{Target}};

		\end{tikzpicture}
	}
	\caption{An example of Class A. The goal is to find the shortest sequence of moves to reach the target state.}
	\label{fig:class-a-example}
\end{figure}

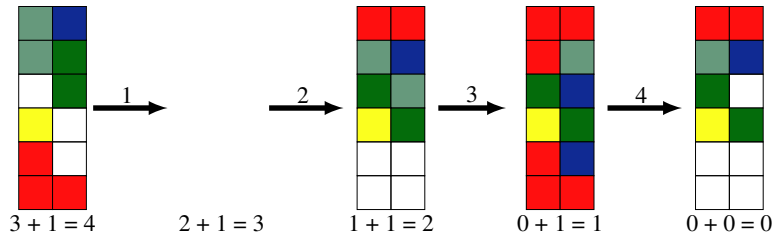
\begin{figure}[htbp]
	\centering
	\resizebox{0.62\textwidth}{!}{
		\begin{tikzpicture}

			\tikzstyle{every node}=[font=\Huge]

			\draw [ fill={rgb,255:red,102; green,153; blue,124} ] (-7.5,-5) rectangle (-6.25,-6.25);
			\draw [ fill={rgb,255:red,16; green,42; blue,147} ] (-5,-5) rectangle (-6.25,-6.25);
			\draw [ fill={rgb,255:red,102; green,153; blue,124} ] (-7.5,-6.25) rectangle (-6.25,-7.5);
			\draw [ fill={rgb,255:red,4; green,108; blue,11} ] (-5,-6.25) rectangle (-6.25,-7.5);
			\draw  (-7.5,-7.5) rectangle (-6.25,-8.75);
			\draw [ fill={rgb,255:red,4; green,108; blue,11} ] (-5,-7.5) rectangle (-6.25,-8.75);
			\draw [ fill={rgb,255:red,251; green,255; blue,31} ] (-7.5,-8.75) rectangle (-6.25,-10);
			\draw  (-5,-8.75) rectangle (-6.25,-10);
			\draw [ fill={rgb,255:red,255; green,15; blue,15} ] (-7.5,-10) rectangle (-6.25,-11.25);
			\draw  (-5,-10) rectangle (-6.25,-11.25);
			\draw [ fill={rgb,255:red,255; green,15; blue,15} ] (-7.5,-11.25) rectangle (-6.25,-12.5);
			\draw [ fill={rgb,255:red,255; green,15; blue,15} ] (-5,-11.25) rectangle (-6.25,-12.5);
			\draw [ fill={rgb,255:red,255; green,15; blue,15} ] (17.5,-5) rectangle (18.75,-6.25);
			\draw [ fill={rgb,255:red,255; green,15; blue,15} ] (20,-5) rectangle (18.75,-6.25);
			\draw [ fill={rgb,255:red,102; green,153; blue,124} ] (17.5,-6.25) rectangle (18.75,-7.5);
			\draw [ fill={rgb,255:red,16; green,42; blue,147} ] (20,-6.25) rectangle (18.75,-7.5);
			\draw [ fill={rgb,255:red,4; green,108; blue,11} ] (17.5,-7.5) rectangle (18.75,-8.75);
			\draw  (20,-7.5) rectangle (18.75,-8.75);
			\draw [ fill={rgb,255:red,251; green,255; blue,31} ] (17.5,-8.75) rectangle (18.75,-10);
			\draw [ fill={rgb,255:red,4; green,108; blue,11} ] (20,-8.75) rectangle (18.75,-10);
			\draw  (17.5,-10) rectangle (18.75,-11.25);
			\draw  (20,-10) rectangle (18.75,-11.25);
			\draw  (17.5,-11.25) rectangle (18.75,-12.5);
			\draw  (20,-11.25) rectangle (18.75,-12.5);
			\draw [ fill={rgb,255:red,255; green,15; blue,15} ] (5,-5) rectangle (6.25,-6.25);
			\draw [ fill={rgb,255:red,255; green,15; blue,15} ] (7.5,-5) rectangle (6.25,-6.25);
			\draw [ fill={rgb,255:red,102; green,153; blue,124} ] (5,-6.25) rectangle (6.25,-7.5);
			\draw [ fill={rgb,255:red,16; green,42; blue,147} ] (7.5,-6.25) rectangle (6.25,-7.5);
			\draw [ fill={rgb,255:red,4; green,108; blue,11} ] (5,-7.5) rectangle (6.25,-8.75);
			\draw [ fill={rgb,255:red,102; green,153; blue,124} ] (7.5,-7.5) rectangle (6.25,-8.75);
			\draw [ fill={rgb,255:red,251; green,255; blue,31} ] (5,-8.75) rectangle (6.25,-10);
			\draw [ fill={rgb,255:red,4; green,108; blue,11} ] (7.5,-8.75) rectangle (6.25,-10);
			\draw  (5,-10) rectangle (6.25,-11.25);
			\draw  (7.5,-10) rectangle (6.25,-11.25);
			\draw  (5,-11.25) rectangle (6.25,-12.5);
			\draw  (7.5,-11.25) rectangle (6.25,-12.5);
			\draw [ fill={rgb,255:red,255; green,15; blue,15} ] (11.25,-5) rectangle (12.5,-6.25);
			\draw [ fill={rgb,255:red,255; green,15; blue,15} ] (13.75,-5) rectangle (12.5,-6.25);
			\draw [ fill={rgb,255:red,255; green,15; blue,15} ] (11.25,-6.25) rectangle (12.5,-7.5);
			\draw [ fill={rgb,255:red,102; green,153; blue,124} ] (13.75,-6.25) rectangle (12.5,-7.5);
			\draw [ fill={rgb,255:red,4; green,108; blue,11} ] (11.25,-7.5) rectangle (12.5,-8.75);
			\draw [ fill={rgb,255:red,16; green,42; blue,147} ] (13.75,-7.5) rectangle (12.5,-8.75);
			\draw [ fill={rgb,255:red,251; green,255; blue,31} ] (11.25,-8.75) rectangle (12.5,-10);
			\draw [ fill={rgb,255:red,4; green,108; blue,11} ] (13.75,-8.75) rectangle (12.5,-10);
			\draw [ fill={rgb,255:red,255; green,15; blue,15} ] (11.25,-10) rectangle (12.5,-11.25);
			\draw [ fill={rgb,255:red,16; green,42; blue,147} ] (13.75,-10) rectangle (12.5,-11.25);
			\draw [ fill={rgb,255:red,255; green,15; blue,15} ] (11.25,-11.25) rectangle (12.5,-12.5);
			\draw [ fill={rgb,255:red,255; green,15; blue,15} ] (13.75,-11.25) rectangle (12.5,-12.5);
			\draw [->, line width=2mm, >=latex] (-4.75,-8.75) -- (-2,-8.75);
			\draw [->, line width=2mm, >=latex] (1.75,-8.75) -- (4.5,-8.75);
			\draw [->, line width=2mm, >=latex] (8,-8.75) -- (10.5,-8.75);
			\draw [->, line width=2mm, >=latex] (14.25,-8.75) -- (17,-8.75);
			\node [font=\Huge] at (-3.5,-8.25) {1};
			\node [font=\Huge] at (3,-8.25) {2};
			\node [font=\Huge] at (9.25,-8.25) {3};
			\node [font=\Huge] at (15.5,-8.25) {4};
			\node [font=\Huge] at (-6.25,-13) {3 + 1 = 4};
			\node [font=\Huge] at (0,-13) {2 + 1 = 3};
			\node [font=\Huge] at (6.25,-13) {1 + 1 = 2 };
			\node [font=\Huge] at (12.5,-13) {0 + 1 = 1};
			\node [font=\Huge] at (18.75,-13) {0 + 0 = 0};

		\end{tikzpicture}
	}
	\caption{This sequence demonstrates the optimal solution for moving the target piece from its initial state (0) to the final target state (4), with the computed heuristic value for Class A provided below each state.}
	\label{fig:class-a-optimal-sequence}
\end{figure}

\subsection{Class C}

\begin{figure}[h]
	\centering
	\resizebox{0.13\textwidth}{!}{
		\begin{tikzpicture}

			\tikzstyle{every node}=[font=\Huge]

			\draw [ fill={rgb,255:red,240; green,179; blue,45} ] (-21.25,-56.25) rectangle (-20,-57.5);
			\draw [ fill={rgb,255:red,240; green,179; blue,45} ] (-18.75,-56.25) rectangle (-20,-57.5);
			\draw [ fill={rgb,255:red,240; green,179; blue,45} ] (-21.25,-57.5) rectangle (-20,-58.75);
			\draw [ fill={rgb,255:red,249; green,6; blue,6} ] (-18.75,-57.5) rectangle (-20,-58.75);
			\draw [ fill={rgb,255:red,11; green,228; blue,7} ] (-21.25,-58.75) rectangle (-20,-60);
			\draw [ fill={rgb,255:red,249; green,6; blue,6} ] (-18.75,-58.75) rectangle (-20,-60);
			\draw [ fill={rgb,255:red,7; green,224; blue,228} ] (-21.25,-60) rectangle (-20,-61.25);
			\draw [ fill={rgb,255:red,7; green,224; blue,228} ] (-18.75,-60) rectangle (-20,-61.25);
			\draw  (-21.25,-61.25) rectangle (-20,-62.5);
			\draw [ fill={rgb,255:red,238; green,255; blue,0} ] (-18.75,-61.25) rectangle (-20,-62.5);
			\draw [ fill={rgb,255:red,226; green,167; blue,167} ] (-21.25,-62.5) rectangle (-20,-63.75);
			\draw [ fill={rgb,255:red,194; green,25; blue,172} ] (-18.75,-62.5) rectangle (-20,-63.75);
			\draw [ fill={rgb,255:red,12; green,50; blue,12} ] (-21.25,-63.75) rectangle (-20,-65);
			\draw [ fill={rgb,255:red,194; green,25; blue,172} ] (-18.75,-63.75) rectangle (-20,-65);
			\draw [ fill={rgb,255:red,0; green,42; blue,255} ] (-21.25,-65) rectangle (-20,-66.25);
			\draw [ fill={rgb,255:red,194; green,25; blue,172} ] (-18.75,-65) rectangle (-20,-66.25);
			\draw  (-17.5,-56.25) rectangle (-16.25,-57.5);
			\draw  (-15,-56.25) rectangle (-16.25,-57.5);
			\draw  (-17.5,-57.5) rectangle (-16.25,-58.75);
			\draw  (-15,-57.5) rectangle (-16.25,-58.75);
			\draw  (-17.5,-58.75) rectangle (-16.25,-60);
			\draw  (-15,-58.75) rectangle (-16.25,-60);
			\draw  (-17.5,-60) rectangle (-16.25,-61.25);
			\draw  (-15,-60) rectangle (-16.25,-61.25);
			\draw  (-17.5,-61.25) rectangle (-16.25,-62.5);
			\draw  (-15,-61.25) rectangle (-16.25,-62.5);
			\draw  (-17.5,-62.5) rectangle (-16.25,-63.75);
			\draw  (-15,-62.5) rectangle (-16.25,-63.75);
			\draw  (-17.5,-63.75) rectangle (-16.25,-65);
			\draw  (-15,-63.75) rectangle (-16.25,-65);
			\draw [ fill={rgb,255:red,249; green,6; blue,6} ] (-17.5,-65) rectangle (-16.25,-66.25);
			\draw [ fill={rgb,255:red,249; green,6; blue,6} ] (-15,-65) rectangle (-16.25,-66.25);
			\node [font=\Huge] at (-20,-67) {\scalebox{1.5}{Initial}};
			\node [font=\Huge] at (-16.25,-67) {\scalebox{1.5}{Target}};

		\end{tikzpicture}
	}
	\caption{An example of Class C. The goal is to find the shortest sequence of moves to reach the target state.}
	\label{fig:class-c-example}
\end{figure}
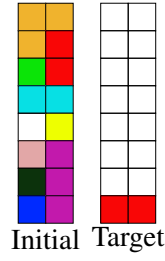

\section{Full Consistency Proof for the Class B Heuristic}\label{app:class-b-proof}

This appendix provides the full proof of Theorem~\ref{thm:hB_consistency} (Consistency of $h_B$); the statement and proof sketch appear in the Class~B subsection of the main text.

\begin{proof}
	\textbf{Condition 1 (Goal value).}

	In any goal state $s_{\text{goal}}$, the goal piece already occupies the target location, hence $d(s_{\text{goal}})=0$.
	No piece blocks the path, so $P_B(s_{\text{goal}}) = \varnothing$. Therefore,
	\[
		h_B(s_{\text{goal}}) = \left\lceil \frac{d(s_{\text{goal}})}{k(s_{\text{goal}})} \right\rceil
		+ \sum_{p \in P_B(s_{\text{goal}})} \left\lceil \frac{g_p(s_{\text{goal}})}{k(s_{\text{goal}})} \right\rceil = 0.
	\]

	\textbf{Condition 2 (One-step change bound).}

	Let $s \to s'$ be any valid move. We show $h_B(s) - h_B(s') \leq 1$ by cases.

	\medskip
	\emph{\textbf{Case A} (Move the goal piece toward the target).}
	In one move, the goal piece can advance by at most $k(s)$ cells along the path. Thus
	$0 \leq d(s) - d(s') \leq k(s)$, so
	$\lceil d(s)/k(s) \rceil - \lceil d(s')/k(s) \rceil \leq 1$,
	while all $g_p$ terms are unchanged. Hence $h_B(s) - h_B(s') \leq 1$.

	\medskip
	\emph{\textbf{Case B} (Move a single blocking piece on the path).}
	A single move can remove from the path at most $k(s)$ occupied units of the moved piece $p$. Thus
	$0 \leq g_p(s) - g_p(s') \leq k(s)$, and so
	$\lceil g_p(s)/k(s) \rceil - \lceil g_p(s')/k(s) \rceil \leq 1$,
	with all other terms unchanged. Therefore $h_B(s) - h_B(s') \leq 1$.

	\medskip
	\emph{\textbf{Case C} (Any other move).}
	The move neither advances the goal piece along the path nor removes occupied path-cells of any blocking piece. Hence neither $\lceil d/k \rceil$ nor any $\lceil g_p/k \rceil$ decreases; thus $h_B(s') \geq h_B(s)$ and $h_B(s) - h_B(s') \leq 0$.

	\medskip
	In all cases, $h_B(s) - h_B(s') \leq 1$. Together with $h_B(s_{\text{goal}})=0$, $h_B$ is consistent.
\end{proof}

\begin{corollary}[Admissibility]
	With unit positive costs and $h_B(s_{\text{goal}}) = 0$, consistency implies admissibility; hence $h_B(s) \leq h^*(s)$ for all $s$.
\end{corollary}

\begin{figure}[htbp]
	\centering
	\includegraphics[width=\textwidth]{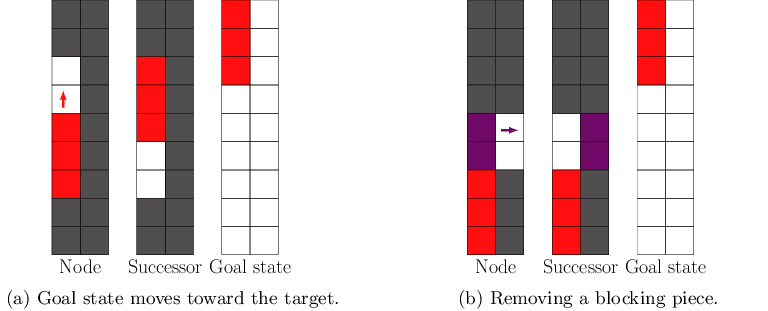}
	\caption{Class B heuristic adjustments.}
	\label{fig:class_b_merged}
\end{figure}

\section{Full Consistency Proof for the Class C Heuristic}\label{app:class-c-proof}

This appendix provides the full case analysis for Theorem~\ref{thm:hC_consistency} (Consistency of $h_C$); the statement and proof sketch appear in Section~\ref{subsec:class-c}.

\begin{proof}
	\textbf{Condition 1: Heuristic at the Goal State.}

	In any goal state $s_{\text{goal}}$, the goal piece occupies the target location with the correct orientation and frame side. Hence
	$d_v(s_{\text{goal}}) = 0$, $R(s_{\text{goal}}) = 0$, $F(s_{\text{goal}}) = 0$, $M(s_{\text{goal}}) = 0$.
	Therefore $h_C(s_{\text{goal}}) = 2(0 + 0 + 0) - 0 = 0$.

	\textbf{Condition 2: The Triangle Inequality.}

	Recall $h_C(s) = 2(d_v(s) + R(s) + F(s)) - M(s)$.
	We need to show $h_C(s) - h_C(s') \leq 1$ for any single transition $s \to s'$.

	\emph{\textbf{Case A} (Clearing the front when blocked).}
	The move clears the front-facing unit. Thus $d_v = d_v'$ and $R$, $F$ are unchanged, while $M'=1$ and $M=0$. Hence
	$h_C(s) - h_C(s') = -(M - M') = 1$.

	\emph{\textbf{Case B} (Advance the goal piece by one unit when ready).}
	One advance reduces vertical distance by exactly one: $d_v' = d_v - 1$. Readiness is lost: $M=1$, $M'=0$. Therefore
	$h_C(s) - h_C(s') = 2(d_v - (d_v-1)) - (1 - 0) = 2 - 1 = 1$.

	\emph{\textbf{Case C} (Eliminate a frame-side discrepancy).}
	$F$ decreases by~1 and readiness is gained ($M'=1$, $M=0$), giving $2(F-F')-(M'-M)=2-1=1$.

	\emph{\textbf{Case D} (Final horizontal placement rotation requirement).}
	$R$ decreases by~1 and readiness is gained ($M'=1$, $M=0$), giving $2(R-R')-(M'-M)=2-1=1$.

	\emph{All other moves.}
	Each such case changes at most one term by one and never increases $d_v$ by more than one per move under $k=1$. In each case the net change is at most~1.

	In all possible single moves, $h_C(s) - h_C(s') \leq 1$. Together with $h_C(s_{\text{goal}})=0$, $h_C$ is consistent.
\end{proof}

\begin{figure}[htbp]
    	\centering
		\includegraphics[width=\textwidth]{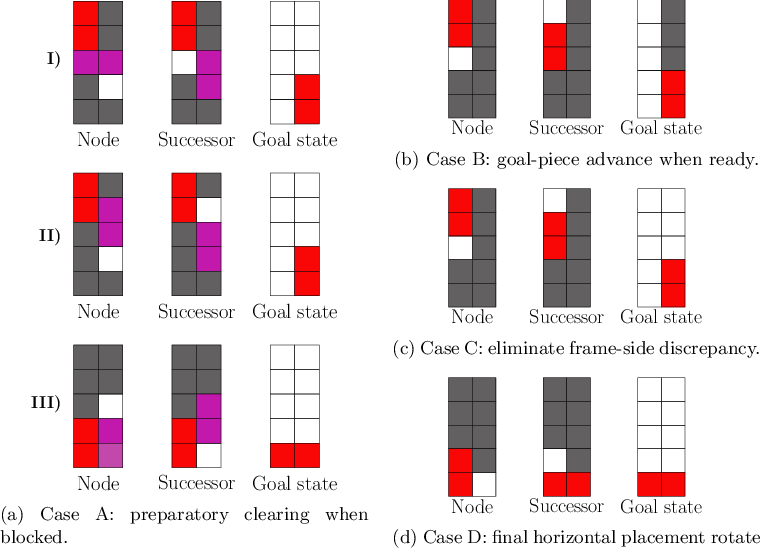}
    	\caption{Four cases of the Class C consistency proof (Cases A--D).}
    	\label{fig:class_c_merged}
\end{figure}

\section{Full Consistency Proof for the Class D Heuristic}\label{app:class-d-proof}

This appendix provides the full case analysis for Theorem~\ref{thm:hD_consistency} (Consistency of $h_D$); the statement and proof sketch appear in Section~\ref{subsec:class-d}.

\begin{proof}
	Condition 1 is identical to the Class~B goal-state argument in
	Appendix~\ref{app:class-b-proof}, Condition~1, with $d(s_{\text{goal}})=0$
	and $P_B(s_{\text{goal}})=\varnothing$ replaced by $d_v(s_{\text{goal}})=0$
	and $P_D(s_{\text{goal}})=\varnothing$; hence $h_D(s_{\text{goal}})=0$.
	For Condition 2, the argument is structurally identical to the Class~B
	proof (Cases A--C), with two substitutions: (a) the raw cell-count
	distance $d$ and its ceiling bound
	$\lceil d/k\rceil - \lceil d'/k\rceil \le 1$ are replaced by the vertical
	row distance $d_v$ with $k=1$ fixed, so the ceiling degenerates to
	$d_v - d_v' \le 1$ directly; (b) by the Crossing Constraint Corollary, no piece
	can cross columns when $k<n$, so no orientation-change or frame-side terms
	arise, and the three cases remain exhaustive. Cases A--C then follow the same
	ceiling argument as the Class~B proof with $k=1$.
\end{proof}

\section{Full Consistency Proof for the Class E Heuristic}\label{app:class-e-proof}

This appendix provides the full case analysis for Theorem~\ref{thm:hE_consistency} (Consistency of $h_E$); the statement and proof sketch appear in Section~\ref{subsec:class-e}.

\begin{proof}
	\textbf{Condition 1: Heuristic at the Goal State.}
	In a goal state $s_{\text{goal}}$, $d_v(s_{\text{goal}}) = 0$, $P_E(s_{\text{goal}}) = \varnothing$, and $\delta(s_{\text{goal}}) = 0$. Therefore $h_E(s_{\text{goal}}) = 0 + 0 + 0 = 0$.

	\textbf{Condition 2: The Triangle Inequality.}
	We need to show $h_E(s) - h_E(s') \leq 1$ for any single move $s \to s'$.

	\begin{itemize}
		\item \textbf{Case A:} The goal piece moves one unit vertically toward the target. The goal piece advances at most one row per move under the two-vacancy bound, so $d_v(s)-d_v(s')\le 1$ and all other terms are unchanged.

		\item \textbf{Case B:} The goal piece slides past a size-1 blocking piece.
		      By the two vacancies, a legal slide bypasses the single-cell blocker without removing it; this reduces the vertical offset by two: $d_v(s') = d_v(s) - 2$. The bypass requires the ``pointing'' precondition, so $\delta(s) = 1$ before the slide and $\delta(s') = 0$ after. No $\lceil g_p/2 \rceil$ term changes. Therefore
		      \[
			      h_E(s) - h_E(s') = (d_v(s) - d_v(s')) - (\delta(s) - \delta(s')) = (2 - 1) = 1 \leq 1.
		      \]

		\item \textbf{Case C:} Remove up to $k = 2$ occupied cells of a blocking piece with size $\geq 2$. One move clears at most $k=2$ occupied cells of a single blocking piece, so $\lceil g_p(s)/2\rceil-\lceil g_p(s')/2\rceil \le 1$ and $d_v$ is unchanged.

		\item \textbf{Case D:} Activate the bypass condition $\delta$. $d_v$ and every $\lceil g_p/2\rceil$ term are unchanged, but $\delta$ toggles $0\to1$ instead of $1\to0$, giving $h_E(s)-h_E(s')=-1\le1$.

		\item \textbf{All other moves.} No term decreases, so $h_E(s)-h_E(s')\le 0 \le 1$.
	\end{itemize}

	In every possible single move, $h_E(s) - h_E(s') \leq 1$. Together with $h_E(s_{\text{goal}}) = 0$, $h_E$ is consistent.
\end{proof}

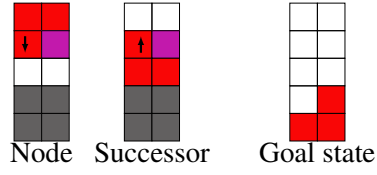
\begin{figure}[h]
	\centering
	\resizebox{0.3\textwidth}{!}{
		\begin{tikzpicture}

			\tikzstyle{every node}=[font=\LARGE]

			\draw [ fill={rgb,255:red,249; green,6; blue,6} ] (-2.5,22.5) rectangle (-1.25,21.25);
			\draw [ fill={rgb,255:red,249; green,6; blue,6} ] (0,22.5) rectangle (-1.25,21.25);
			\draw [ fill={rgb,255:red,249; green,6; blue,6} ] (-2.5,21.25) rectangle (-1.25,20);
			\draw [ fill={rgb,255:red,194; green,25; blue,172} ] (0,21.25) rectangle (-1.25,20);
			\draw  (-2.5,20) rectangle (-1.25,18.75);
			\draw  (0,20) rectangle (-1.25,18.75);
			\draw [ fill={rgb,255:red,98; green,96; blue,96} ] (-2.5,18.75) rectangle (-1.25,17.5);
			\draw [ fill={rgb,255:red,98; green,96; blue,96} ] (0,18.75) rectangle (-1.25,17.5);
			\draw [ fill={rgb,255:red,98; green,96; blue,96} ] (-2.5,17.5) rectangle (-1.25,16.25);
			\draw [ fill={rgb,255:red,98; green,96; blue,96} ] (0,17.5) rectangle (-1.25,16.25);
			\draw  (2.5,22.5) rectangle (3.75,21.25);
			\draw  (5,22.5) rectangle (3.75,21.25);
			\draw [ fill={rgb,255:red,249; green,6; blue,6} ] (2.5,21.25) rectangle (3.75,20);
			\draw [ fill={rgb,255:red,194; green,25; blue,172} ] (5,21.25) rectangle (3.75,20);
			\draw [ fill={rgb,255:red,249; green,6; blue,6} ] (2.5,20) rectangle (3.75,18.75);
			\draw [ fill={rgb,255:red,249; green,6; blue,6} ] (5,20) rectangle (3.75,18.75);
			\draw [ fill={rgb,255:red,98; green,96; blue,96} ] (2.5,18.75) rectangle (3.75,17.5);
			\draw [ fill={rgb,255:red,98; green,96; blue,96} ] (5,18.75) rectangle (3.75,17.5);
			\draw [ fill={rgb,255:red,98; green,96; blue,96} ] (2.5,17.5) rectangle (3.75,16.25);
			\draw [ fill={rgb,255:red,98; green,96; blue,96} ] (5,17.5) rectangle (3.75,16.25);
			\draw  (10,22.5) rectangle (11.25,21.25);
			\draw  (12.5,22.5) rectangle (11.25,21.25);
			\draw  (10,21.25) rectangle (11.25,20);
			\draw  (12.5,21.25) rectangle (11.25,20);
			\draw  (10,20) rectangle (11.25,18.75);
			\draw  (12.5,20) rectangle (11.25,18.75);
			\draw  (10,18.75) rectangle (11.25,17.5);
			\draw [ fill={rgb,255:red,249; green,6; blue,6} ] (12.5,18.75) rectangle (11.25,17.5);
			\draw [ fill={rgb,255:red,249; green,6; blue,6} ] (10,17.5) rectangle (11.25,16.25);
			\draw [ fill={rgb,255:red,249; green,6; blue,6} ] (12.5,17.5) rectangle (11.25,16.25);
			\node [font=\Huge] at (-1.25,15.75) {\scalebox{1.5}{Node}};
			\node [font=\Huge] at (3.75,15.75) {\scalebox{1.5}{Successor}};
			\node [font=\Huge] at (11.25,15.75) {\scalebox{1.5}{Goal state}};
			\draw [->, line width=1mm, >=latex] (-2,21) -- (-2,20.25);
			\draw [->, line width=1mm, >=latex] (3.25,20.25) -- (3.25,21);

		\end{tikzpicture}
	}
	\caption{Example of the goal piece sliding past a blocking piece of size 1, demonstrating the corresponding decrease in heuristic estimation.}
	\label{fig:class-e-bypass-slide}
\end{figure}

\section{Per-Class Experimental Results}\label{app:per-class-results}

This appendix collects the per-instance comparison figures for each of the seven puzzle classes. The condensed analysis for each class appears in the main body; the figures below provide the full per-instance breakdowns of EBF and node expansions.

\subsection{Class A Results}

\begin{figure}[H]
	\centering
	\includegraphics[scale=0.65]{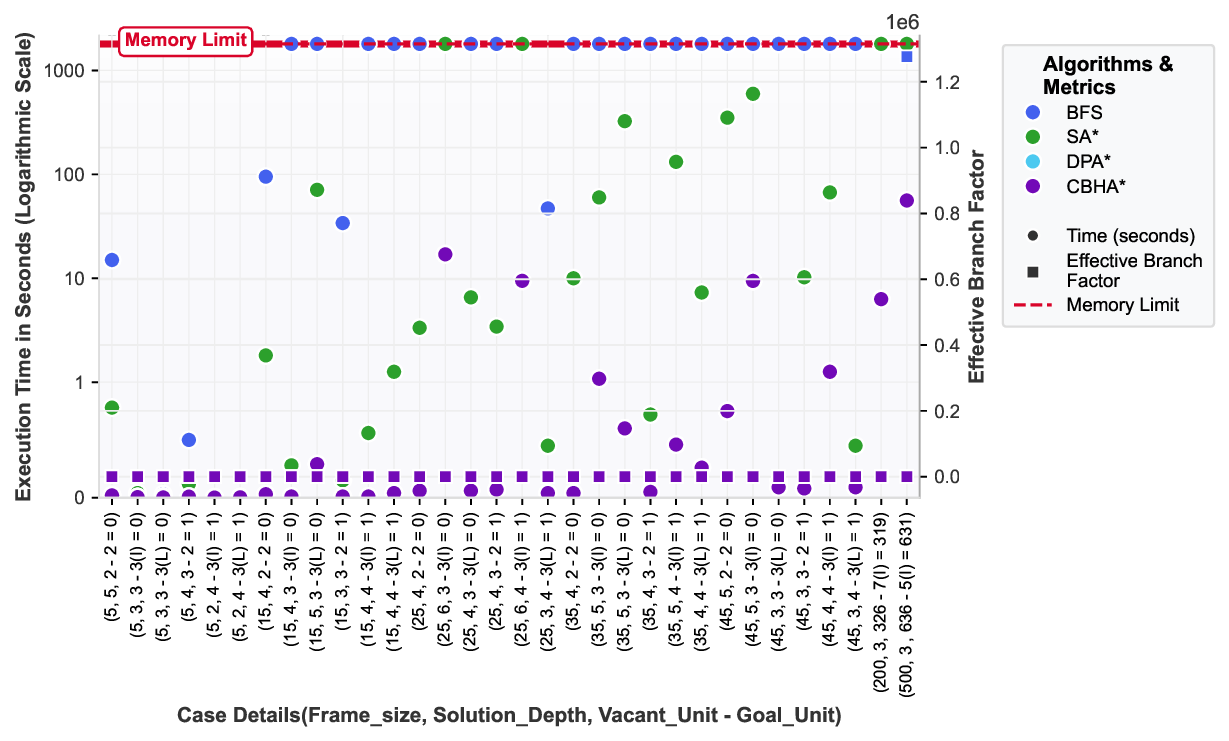}
	\caption{Comparison of execution times and effective branch factors for different algorithms (BFS, SA*, DPA*, and CBHA*) across various case scenarios of Class A.}
	\label{fig:results-class-a}
\end{figure}

\subsection{Class B Results}

\begin{figure}[H]
	\centering
	\includegraphics[scale=0.65]{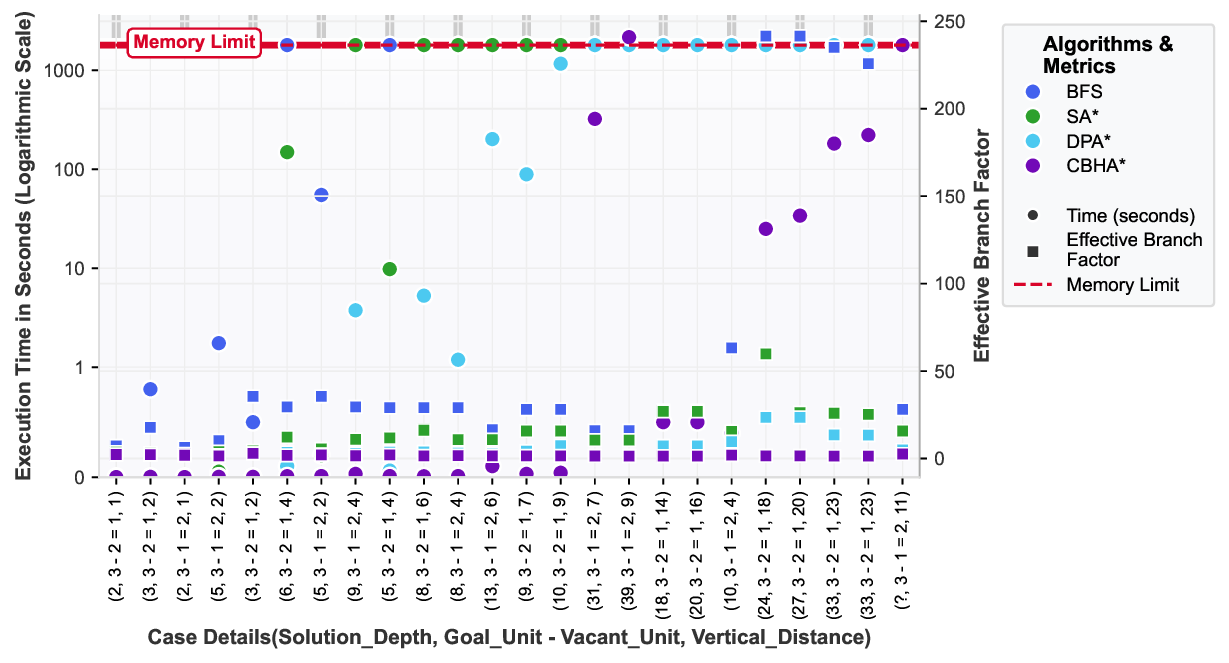}
	\caption{Comparison of execution times and effective branch factors for different algorithms (BFS, SA*, DPA*, and CBHA*) across various case scenarios of Class B.}
	\label{fig:results-class-b}
\end{figure}

\subsection{Class C Results}

\begin{figure}[H]
	\centering
	\includegraphics[scale=0.65]{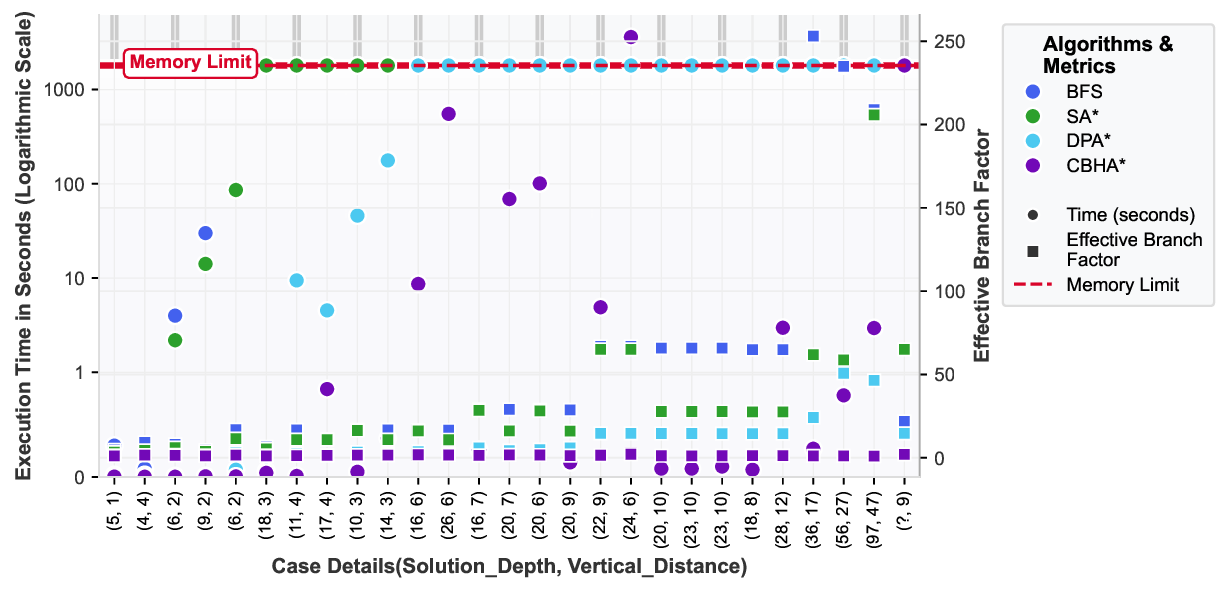}
	\caption{Comparison of execution times and effective branch factors for different algorithms (BFS, SA*, DPA*, and CBHA*) across various case scenarios of Class C.}
	\label{fig:results-class-c}
\end{figure}

\subsection{Class D Results}

\begin{figure}[H]
	\centering
	\includegraphics[scale=0.65]{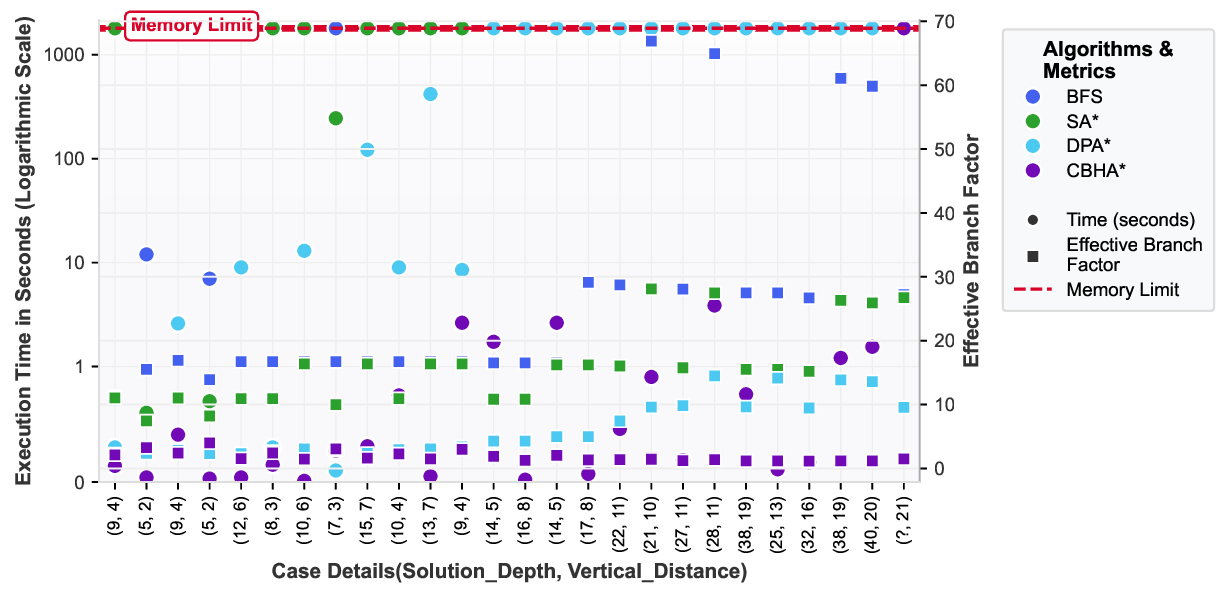}
	\caption{Comparison of execution times and effective branch factors for different algorithms (BFS, SA*, DPA*, and CBHA*) across various case scenarios of Class D.}
	\label{fig:results-class-d}
\end{figure}

\subsection{Class E Results}

\begin{figure}[H]
	\centering
	\includegraphics[scale=0.65]{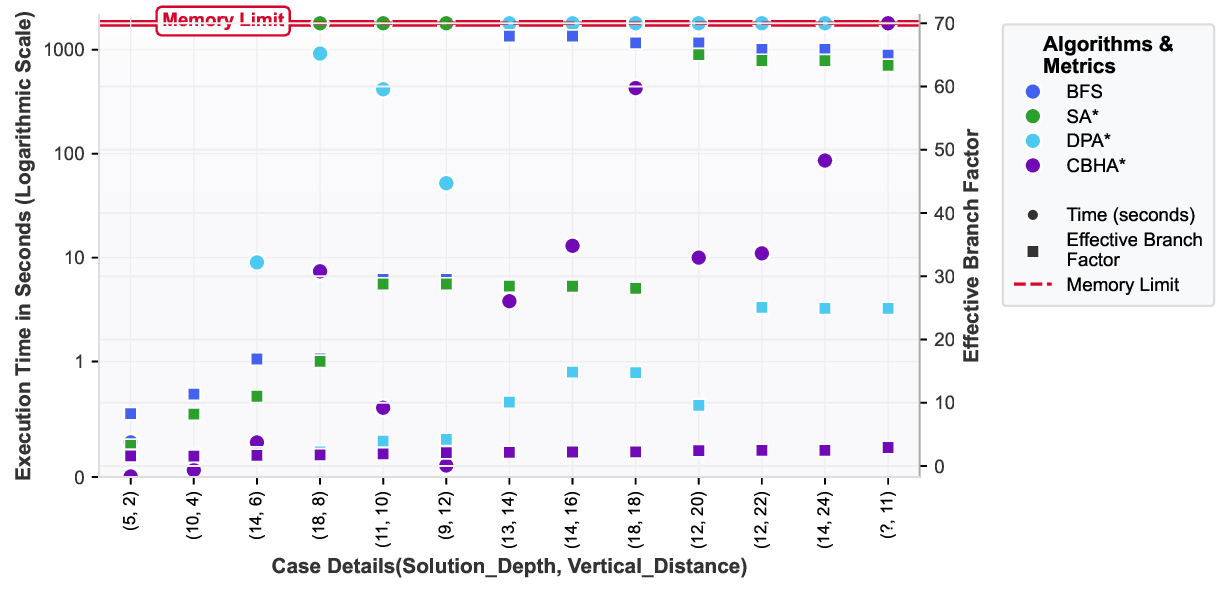}
	\caption{Comparison of execution times and effective branch factors for different algorithms (BFS, SA*, DPA*, and CBHA*) across various case scenarios of Class E.}
	\label{fig:results-class-e}
\end{figure}

\subsection{Class F Results}

\begin{figure}[H]
	\centering
	\includegraphics[scale=0.7]{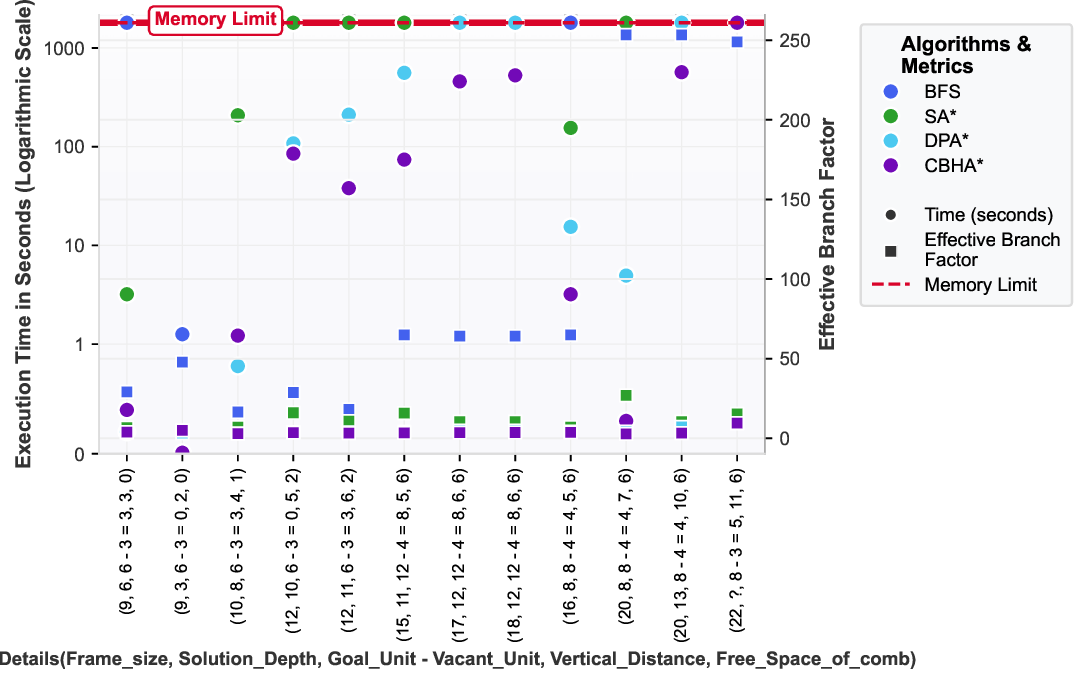}
	\caption{Comparison of execution times and effective branch factors for different algorithms (BFS, SA*, DPA*, and CBHA*) across various case scenarios of Class F.}
	\label{fig:results-class-f}
\end{figure}

\subsection{Class G Results}

\begin{figure}[H]
	\centering
	\includegraphics[scale=0.8]{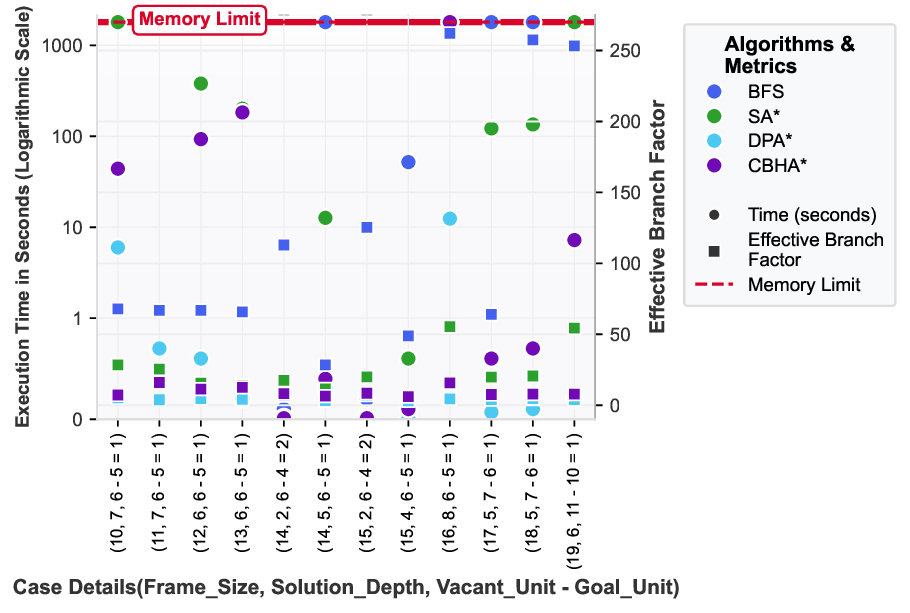}
	\caption{Comparison of execution times and effective branch factors for different algorithms (BFS, SA*, DPA*, and CBHA*) across various case scenarios of Class G.}
	\label{fig:results-class-g}
\end{figure}

\section{Extended Worked Example: Optimal Move Sequence for an $8\times 2$ Instance}\label{app:worked-examples}

\begin{figure}[htbp]
	\centering
	\resizebox{0.65\textwidth}{!}{
		\begin{tikzpicture}

			\tikzstyle{every node}=[font=\LARGE]

			\draw [ fill={rgb,255:red,228; green,7; blue,7} ] (5,13.5) rectangle (3.75,12.25);
			\draw [ fill={rgb,255:red,228; green,7; blue,7} ] (6.25,13.5) rectangle (5,12.25);
			\draw [ fill={rgb,255:red,228; green,7; blue,7} ] (5,12.25) rectangle (3.75,11);
			\draw [ fill={rgb,255:red,228; green,136; blue,7} ] (6.25,12.25) rectangle (5,11);
			\draw [ fill={rgb,255:red,228; green,136; blue,7} ] (6.25,11) rectangle (5,9.75);
			\draw [ fill={rgb,255:red,228; green,136; blue,7} ] (6.25,9.75) rectangle (5,8.5);
			\draw [ fill={rgb,255:red,33; green,228; blue,7} ] (5,7.25) rectangle (3.75,6);
			\draw [ fill={rgb,255:red,33; green,228; blue,7} ] (5,6) rectangle (3.75,4.75);
			\draw [ fill={rgb,255:red,33; green,228; blue,7} ] (5,4.75) rectangle (3.75,3.5);
			\draw [ fill={rgb,255:red,33; green,7; blue,228} ] (5,9.75) rectangle (3.75,8.5);
			\draw [ fill={rgb,255:red,33; green,7; blue,228} ] (5,8.5) rectangle (3.75,7.25);
			\draw [ fill={rgb,255:red,7; green,224; blue,228} ] (6.25,6) rectangle (5,4.75);
			\draw [ fill={rgb,255:red,7; green,224; blue,228} ] (6.25,4.75) rectangle (5,3.5);
			\draw  (3.75,11) rectangle (5,9.75);
			\draw  (6.25,8.5) rectangle (5,7.25);
			\draw  (6.25,7.25) rectangle (5,6);
			\draw [->, line width=1mm, >=latex] (7,8.5) -- (8,8.5);
			\draw [ fill={rgb,255:red,228; green,7; blue,7} ] (10,13.5) rectangle (8.75,12.25);
			\draw [ fill={rgb,255:red,228; green,7; blue,7} ] (11.25,13.5) rectangle (10,12.25);
			\draw [ fill={rgb,255:red,228; green,7; blue,7} ] (10,12.25) rectangle (8.75,11);
			\draw [ fill={rgb,255:red,228; green,136; blue,7} ] (11.25,12.25) rectangle (10,11);
			\draw [ fill={rgb,255:red,228; green,136; blue,7} ] (11.25,11) rectangle (10,9.75);
			\draw [ fill={rgb,255:red,228; green,136; blue,7} ] (11.25,9.75) rectangle (10,8.5);
			\draw [ fill={rgb,255:red,33; green,7; blue,228} ] (11.25,8.5) rectangle (10,7.25);
			\draw [ fill={rgb,255:red,33; green,7; blue,228} ] (11.25,7.25) rectangle (10,6);
			\draw [ fill={rgb,255:red,7; green,224; blue,228} ] (11.25,6) rectangle (10,4.75);
			\draw [ fill={rgb,255:red,7; green,224; blue,228} ] (11.25,4.75) rectangle (10,3.5);
			\draw [ fill={rgb,255:red,33; green,224; blue,7} ] (10,7.25) rectangle (8.75,6);
			\draw [ fill={rgb,255:red,33; green,224; blue,7} ] (10,6) rectangle (8.75,4.75);
			\draw [ fill={rgb,255:red,33; green,224; blue,7} ] (10,4.75) rectangle (8.75,3.5);
			\draw  (8.75,11) rectangle (10,9.75);
			\draw  (8.75,9.75) rectangle (10,8.5);
			\draw  (8.75,8.5) rectangle (10,7.25);
			\draw [->, line width=1mm, >=latex] (12,8.5) -- (13,8.5);
			\draw [ fill={rgb,255:red,228; green,7; blue,7} ] (15,13.5) rectangle (13.75,12.25);
			\draw [ fill={rgb,255:red,228; green,7; blue,7} ] (16.25,13.5) rectangle (15,12.25);
			\draw [ fill={rgb,255:red,228; green,7; blue,7} ] (15,12.25) rectangle (13.75,11);
			\draw [ fill={rgb,255:red,228; green,136; blue,7} ] (16.25,12.25) rectangle (15,11);
			\draw [ fill={rgb,255:red,228; green,136; blue,7} ] (16.25,11) rectangle (15,9.75);
			\draw [ fill={rgb,255:red,228; green,136; blue,7} ] (16.25,9.75) rectangle (15,8.5);
			\draw [ fill={rgb,255:red,33; green,7; blue,228} ] (16.25,8.5) rectangle (15,7.25);
			\draw [ fill={rgb,255:red,33; green,7; blue,228} ] (16.25,7.25) rectangle (15,6);
			\draw [ fill={rgb,255:red,7; green,224; blue,228} ] (16.25,6) rectangle (15,4.75);
			\draw [ fill={rgb,255:red,7; green,224; blue,228} ] (16.25,4.75) rectangle (15,3.5);
			\draw [ fill={rgb,255:red,33; green,228; blue,7} ] (15,11) rectangle (13.75,9.75);
			\draw [ fill={rgb,255:red,33; green,228; blue,7} ] (15,9.75) rectangle (13.75,8.5);
			\draw [ fill={rgb,255:red,33; green,228; blue,7} ] (15,8.5) rectangle (13.75,7.25);
			\draw  (13.75,7.25) rectangle (15,6);
			\draw  (13.75,6) rectangle (15,4.75);
			\draw  (13.75,4.75) rectangle (15,3.5);
			\draw [->, line width=1mm, >=latex] (17,8.5) -- (18.25,8.5);
			\draw [ fill={rgb,255:red,228; green,7; blue,7} ] (20,13.5) rectangle (18.75,12.25);
			\draw [ fill={rgb,255:red,228; green,7; blue,7} ] (21.25,13.5) rectangle (20,12.25);
			\draw [ fill={rgb,255:red,228; green,7; blue,7} ] (20,12.25) rectangle (18.75,11);
			\draw [ fill={rgb,255:red,228; green,136; blue,7} ] (21.25,12.25) rectangle (20,11);
			\draw [ fill={rgb,255:red,228; green,136; blue,7} ] (21.25,11) rectangle (20,9.75);
			\draw [ fill={rgb,255:red,228; green,136; blue,7} ] (21.25,9.75) rectangle (20,8.5);
			\draw [ fill={rgb,255:red,33; green,7; blue,228} ] (21.25,8.5) rectangle (20,7.25);
			\draw [ fill={rgb,255:red,33; green,7; blue,228} ] (21.25,7.25) rectangle (20,6);
			\draw [ fill={rgb,255:red,7; green,224; blue,228} ] (20,7.25) rectangle (18.75,6);
			\draw [ fill={rgb,255:red,7; green,224; blue,228} ] (20,6) rectangle (18.75,4.75);
			\draw [ fill={rgb,255:red,33; green,228; blue,7} ] (18.75,11) rectangle (20,9.75);
			\draw [ fill={rgb,255:red,33; green,228; blue,7} ] (18.75,9.75) rectangle (20,8.5);
			\draw [ fill={rgb,255:red,33; green,228; blue,7} ] (18.75,8.5) rectangle (20,7.25);
			\draw  (21.25,6) rectangle (20,4.75);
			\draw  (21.25,4.75) rectangle (20,3.5);
			\draw  (18.75,4.75) rectangle (20,3.5);
			\draw [->, line width=1mm, >=latex] (22,8.5) -- (23.25,8.5);
			\draw [ fill={rgb,255:red,228; green,7; blue,7} ] (26.25,6) rectangle (25,4.75);
			\draw [ fill={rgb,255:red,228; green,7; blue,7} ] (26.25,4.75) rectangle (25,3.5);
			\draw [ fill={rgb,255:red,228; green,7; blue,7} ] (25,4.75) rectangle (23.75,3.5);
			\draw [ fill={rgb,255:red,7; green,224; blue,228} ] (25,7.25) rectangle (23.75,6);
			\draw [ fill={rgb,255:red,7; green,224; blue,228} ] (25,6) rectangle (23.75,4.75);
			\draw [ fill={rgb,255:red,33; green,7; blue,228} ] (26.25,7.25) rectangle (25,6);
			\draw [ fill={rgb,255:red,33; green,7; blue,228} ] (26.25,8.5) rectangle (25,7.25);
			\draw [ fill={rgb,255:red,33; green,228; blue,7} ] (25,11) rectangle (23.75,9.75);
			\draw [ fill={rgb,255:red,33; green,228; blue,7} ] (25,9.75) rectangle (23.75,8.5);
			\draw [ fill={rgb,255:red,33; green,228; blue,7} ] (25,8.5) rectangle (23.75,7.25);
			\draw [ fill={rgb,255:red,228; green,136; blue,7} ] (26.25,12.25) rectangle (25,11);
			\draw [ fill={rgb,255:red,228; green,136; blue,7} ] (26.25,11) rectangle (25,9.75);
			\draw [ fill={rgb,255:red,228; green,136; blue,7} ] (26.25,9.75) rectangle (25,8.5);
			\draw  (26.25,13.5) rectangle (25,12.25);
			\draw  (23.75,13.5) rectangle (25,12.25);
			\draw  (23.75,12.25) rectangle (25,11);
			\draw [ fill={rgb,255:red,7; green,224; blue,228} ] (30,13.5) rectangle (28.75,12.25);
			\draw [ fill={rgb,255:red,7; green,224; blue,228} ] (31.25,13.5) rectangle (30,12.25);
			\draw [ fill={rgb,255:red,228; green,136; blue,7} ] (31.25,12.25) rectangle (30,11);
			\draw [ fill={rgb,255:red,228; green,136; blue,7} ] (31.25,11) rectangle (30,9.75);
			\draw [ fill={rgb,255:red,228; green,136; blue,7} ] (31.25,9.75) rectangle (30,8.5);
			\draw [ fill={rgb,255:red,33; green,228; blue,7} ] (28.75,11) rectangle (30,9.75);
			\draw [ fill={rgb,255:red,33; green,228; blue,7} ] (28.75,9.75) rectangle (30,8.5);
			\draw [ fill={rgb,255:red,33; green,228; blue,7} ] (28.75,8.5) rectangle (30,7.25);
			\draw [ fill={rgb,255:red,33; green,7; blue,228} ] (31.25,8.5) rectangle (30,7.25);
			\draw [ fill={rgb,255:red,33; green,7; blue,228} ] (31.25,7.25) rectangle (30,6);
			\draw [ fill={rgb,255:red,228; green,7; blue,7} ] (31.25,6) rectangle (30,4.75);
			\draw [ fill={rgb,255:red,228; green,7; blue,7} ] (31.25,4.75) rectangle (30,3.5);
			\draw [ fill={rgb,255:red,228; green,7; blue,7} ] (28.75,4.75) rectangle (30,3.5);
			\draw  (28.75,12.25) rectangle (30,11);
			\draw  (28.75,7.25) rectangle (30,6);
			\draw  (28.75,6) rectangle (30,4.75);
			\draw [->, line width=1mm, >=latex] (26.75,8.5) -- (28,8.5);
			\draw [->, line width=1mm, >=latex] (31.75,8.5) -- (33,8.5);
			\draw [ fill={rgb,255:red,228; green,7; blue,7} ] (35,6) rectangle (33.75,4.75);
			\draw [ fill={rgb,255:red,228; green,7; blue,7} ] (35,4.75) rectangle (33.75,3.5);
			\draw [ fill={rgb,255:red,228; green,7; blue,7} ] (36.25,4.75) rectangle (35,3.5);
			\draw [ fill={rgb,255:red,33; green,7; blue,228} ] (36.25,7.25) rectangle (35,6);
			\draw [ fill={rgb,255:red,33; green,7; blue,228} ] (36.25,8.5) rectangle (35,7.25);
			\draw [ fill={rgb,255:red,33; green,228; blue,7} ] (35,8.5) rectangle (33.75,7.25);
			\draw [ fill={rgb,255:red,33; green,228; blue,7} ] (35,9.75) rectangle (33.75,8.5);
			\draw [ fill={rgb,255:red,33; green,228; blue,7} ] (35,11) rectangle (33.75,9.75);
			\draw [ fill={rgb,255:red,228; green,136; blue,7} ] (36.25,9.75) rectangle (35,8.5);
			\draw [ fill={rgb,255:red,228; green,136; blue,7} ] (36.25,11) rectangle (35,9.75);
			\draw [ fill={rgb,255:red,228; green,136; blue,7} ] (36.25,12.25) rectangle (35,11);
			\draw [ fill={rgb,255:red,7; green,224; blue,228} ] (36.25,13.5) rectangle (35,12.25);
			\draw [ fill={rgb,255:red,7; green,224; blue,228} ] (33.75,13.5) rectangle (35,12.25);
			\draw  (36.25,6) rectangle (35,4.75);
			\draw  (33.75,7.25) rectangle (35,6);
			\draw  (33.75,12.25) rectangle (35,11);
			\node [font=\LARGE] at (7.5,9.25) {1};
			\node [font=\LARGE] at (12.5,9.25) {2};
			\node [font=\LARGE] at (17.5,9.25) {3};
			\node [font=\LARGE] at (22.5,9.25) {4};
			\node [font=\LARGE] at (27.5,9.25) {5};
			\node [font=\LARGE] at (32.5,9.25) {6};

		\end{tikzpicture}
	}
	\caption{This sequence demonstrates the optimal solution for moving the target piece from its initial state (0) to the final target state (6).}
	\label{fig:optimal-move-sequence}
\end{figure}
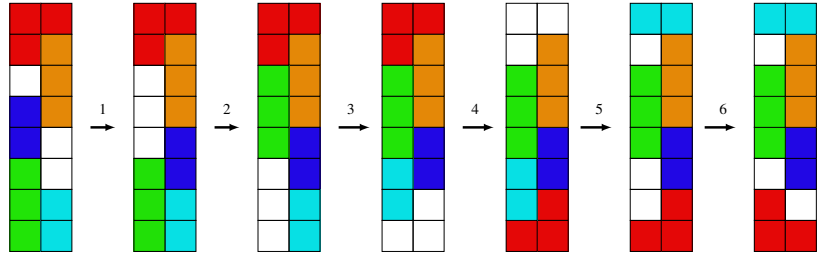

\section{Structural Parallels Between the Kinematic Taxonomy and Physical Constraint Systems}\label{app:sn8-industrial-mapping}

The Flying Block Puzzle functions as a controlled microworld that isolates the
cascading positional dependencies arising in restricted-mobility
environments~\cite{russell2020artificial}. The description list below maps each kinematic
equivalence class to a representative physical constraint system whose structural
properties mirror those formalised in the taxonomy. These correspondences are
descriptive structural analogies motivated by the formal class properties;
direct cross-domain experimental evaluation remains future work, and adapting
the class heuristics to each domain's piece kinematics would require deriving
domain-specific forms of the geometric terms $d_v$, $R$, $F$, and $\delta$,
paralleling standard practices in adaptive navigation planning.
These mappings instantiate the principles of qualitative physics, reasoning about physical systems through structural properties and monotonic relationships without numerical simulation. Here $k$ denotes the number of vacant units, $n$ the goal-piece size, and $s$ the number of free spaces in a comb piece; each entry closes with a representative puzzle instance, cross-referenced to the example figures where available, to concretise the analogy. These mappings treat each class in
isolation; Section~\ref{sn10-1-multi-class-transition-narrative-pipelines} below presents integrated multi-class transition
pipelines showing how a single physical system moves dynamically across
multiple classes.

\begin{description}[style=sameline, leftmargin=1.5em]
    \item[\textbf{Class A (Sparse-Environment Dispatch):}]
    Characterised by a high vacancy-to-size ratio ($k \ge n$) where jumping moves are available.
    \textit{Physical Counterpart:} Autonomous vehicle initialisation in a low-density warehouse.
    \textit{Puzzle Equivalence:} A $6\times 2$ frame with $k=4$ vacant units and a goal-piece size of $n=3$, where the goal piece can jump directly to its target in a single move (Fig.~\ref{fig:class-a-example}).

    \item[\textbf{Class B (Column-Locked Sliding Routing):}]
    Defined by an I-shaped goal piece ($n \ge 3$) under space constraints ($k < n$), restricting the system to sliding moves only.
    \textit{Physical Counterpart:} VLSI channel routing navigating narrow silicon gaps with long rigid wires.
    \textit{Puzzle Equivalence:} An $8\times 2$ frame with $k=1$ and $n=3$, requiring a 14-step sliding sequence to displace two column-locked blockers.

    \item[\textbf{Class C (Rotational Component Assembly):}]
    Arises when an I-shaped goal piece has a size of $n=2$ with exactly $k=1$ vacant unit, which uniquely enables rotation and column-crossing.
    \textit{Physical Counterpart:} High-density PCB assembly where micro-SMD components must rotate on a crowded board to navigate narrow gaps.
    \textit{Puzzle Equivalence:} An $8\times 2$ frame with $k=1$ and $n=2$, requiring a 16-step optimal sequence of coordinated slides and rotations (Fig.~\ref{fig:class-c-example}).

    \item[\textbf{Class D (Confined Arm Manoeuvring):}]
    Formed by an L-shaped goal piece under severe vacancy constraints ($k=1$), where all path-blocking obstacles must be fully cleared.
    \textit{Physical Counterpart:} Articulated elbow joint of a manipulator navigating a turbine hull with no bypass lanes.
    \textit{Puzzle Equivalence:} An $8\times 2$ frame with $k=1$ and an L-shaped goal piece, where bypass moves are impossible, forcing full clearance of path blockers.

    \item[\textbf{Class E (Narrow-Corridor Bypass MAPF):}]
    Occurs with an L-shaped goal piece and $k=2$ vacant units, which allows size-1 blocking obstacles to be bypassed.
    \textit{Physical Counterpart:} Narrow-corridor large-agent Multi-Agent Path Finding, simulating two-agent exchange in a logistics facility with single-unit side bays.
    \textit{Puzzle Equivalence:} An $8\times 2$ frame with $k=2$ and an L-shaped goal piece, where a size-1 blocker is bypassed via a two-vacancy slide (Fig.~\ref{fig:class-e-bypass-slide}).

    \item[\textbf{Class F (Fixed-Orientation Monorail Assembly):}]
    An equivalence class for comb-shaped pieces where the number of free spaces exceeds vacant units ($s > k$). The comb's body is side-pinned, rendering a vertical flip impossible (Theorem~\ref{thm:comb_flip}).
    \textit{Physical Counterpart:} A fixed-orientation chassis part on a constrained monorail assembly line.
    \textit{Puzzle Equivalence:} A comb piece with $s=2$ and $k=1$ where the body is locked to a single column, requiring multi-step coordination for vertical progress.

    \item[\textbf{Class G (Irregular Clutter Bin Picking):}]
    The residual case representing symmetric shapes or comb-shaped pieces with $s \le k$, where standard geometric heuristics are uninformative.
    \textit{Physical Counterpart:} Heterogeneous parts handling in unsorted industrial bins.
    \textit{Puzzle Equivalence:} A comb piece with $s=1 \le k=2$ where the general heuristic $h_G = \lceil d/k\rceil$ provides optimal guidance.
\end{description}

\subsection{Multi-Class Transition Narrative Pipelines}\label{sn10-1-multi-class-transition-narrative-pipelines}

In physical and industrial systems, constraint regimes are rarely static.
A single continuous mission typically moves a planning agent from unconstrained
open spaces through progressively tighter corridors and back again --- a
sequence that corresponds mathematically to traversing multiple kinematic
classes (A--G) within one operational pipeline. Rather than tabulating these
transitions row by row, we set out below, for two integrated engineering
systems, which CBHS class --- and therefore which heuristic --- applies at
each operational phase.

\subsubsection*{Pipeline I: The Fulfillment Pipeline (Autonomous Logistics \& Warehousing Fleet)}
This pipeline simulates an Automated Guided Vehicle (AGV) navigating a high-density warehouse:
\begin{enumerate}[label=\arabic*., leftmargin=1.5em]
    \item \textbf{Initial Dispatch (Class A, $k \ge n$):} The AGV manoeuvres through wide-open loading bays. Multiple open slots allow free repositioning, including jumping moves, so the general heuristic $h_A = h_{\mathrm{general}}$ applies. \textit{Puzzle equivalence:} a $6{\times}2$ frame, $k=4$, $n=3$; the goal piece jumps to target in 1 move.
    \item \textbf{Narrow Aisle Access (Class B, $k < n$, I-shaped):} Upon entering a high-density shelving lane, physical boundaries contract. The vehicle must displace blocking pallets sequentially, guided by $h_B = \lceil d/k \rceil + \sum \lceil g_p/k \rceil$. \textit{Puzzle equivalence:} an $8{\times}2$ frame, $k=1$, $n=3$ I-shaped piece.
    \item \textbf{Intersection Turn (Class C, $k=1$, $n=2$, I-shaped):} At a tight aisle junction, the AGV rotates $90^\circ$ and crosses column lanes. The planning layer switches to $h_C = 2(d_v + R + F) - M$. \textit{Puzzle equivalence:} an $8{\times}2$ frame, $k=1$, $n=2$ I-shaped piece.
    \item \textbf{Narrow-Corridor Bypass (Class E, $k=2$, L-shaped):} The AGV uses single-unit side bays to bypass the blocker, matching $h_E = d_v + \sum \lceil g_p/2 \rceil - \delta$. \textit{Puzzle equivalence:} an $8{\times}2$ frame, $k=2$, L-shaped piece.
    \item \textbf{Unsorted Bin Picking (Class G, Residual):} The AGV enters a high-clutter sorting area. The planning layer relies on the residual heuristic $h_G = \lceil d/k \rceil$. \textit{Puzzle equivalence:} a comb piece with $s=1 \le k=2$.
\end{enumerate}

\subsubsection*{Pipeline II: The Turbine Solder Mission (Articulated Robotic Hull Inspection)}
This pipeline describes an articulated manipulator arm executing a precision hull inspection:
\begin{enumerate}[label=\arabic*., leftmargin=1.5em]
    \item \textbf{Deployment Cavity (Class A, $k \ge n$):} The manipulator arm initialises in a spacious turbine intake chamber, so $h_A = h_{\mathrm{general}}$ applies. \textit{Puzzle equivalence:} a $6{\times}2$ frame, $k=4$, $n=3$.
    \item \textbf{Rib-Clearance Traverse (Class D, $k=1$, L-shaped):} The arm's L-shaped elbow joint navigates tight rib structures; all blocking components must be sequentially cleared, guided by $h_D = d_v + \sum g_p$. \textit{Puzzle equivalence:} an $8{\times}2$ frame, $k=1$, L-shaped piece.
    \item \textbf{Protrusion Bypass (Class E, $k=2$, L-shaped):} A 2-vacancy slide bypass manoeuvre is used, matching $h_E = d_v + \sum \lceil g_p/2 \rceil - \delta$. \textit{Puzzle equivalence:} an $8{\times}2$ frame, $k=2$, L-shaped piece.
    \item \textbf{Track Solder Operation (Class F, $s > k$, Comb):} The comb-like end-effector is pinned against a circular track ($s > k$), so the system switches to $h_F = \lceil d/k \rceil + \sum \lceil g_i/k \rceil$. \textit{Puzzle equivalence:} a comb piece with $s=2 > k=1$.
\end{enumerate}

\subsection{Detailed Operational Transitions}\label{sn10-2-detailed-operational-transitions}

\paragraph{Logistics AGV Case Study.}
Consider an automated warehouse system operating in a high-density logistics
facility. The AGV begins its mission in the dispatch zone (\textbf{Class~A}
regime), where the ratio of vacant space to AGV footprint is high ($k \ge n$).
Here, planning is cheap: $h_{\mathrm{general}}$ provides accurate guidance
because obstacles are sparse and the goal piece can reach its target via a
single jumping move.

As the AGV enters a storage rack corridor, the physical boundaries contract,
entering the \textbf{Class~B} regime ($k < n$, I-shaped goal piece). The
planning layer dynamically switches its heuristic from $h_{\mathrm{general}}$
to $h_B$, which accounts for column-locked, sliding-only dynamics and estimates
clearance cost as $\lceil d/k \rceil + \sum_{p} \lceil g_p/k \rceil$. Without
this switch, the general heuristic would plateau and provide no discriminating
guidance.

Upon reaching the end of the aisle, the AGV must navigate a narrow T-junction
(\textbf{Class~C} regime), requiring a 90-degree rotation and a column-crossing
manoeuvre. The planning layer immediately applies
$h_C = 2(d_v + R + F) - M$, penalising orientation discrepancy ($R$),
frame-side mismatch ($F$), and the alternating clear-then-advance constraint.
This transition illustrates why static single-class heuristics fail: a planner
using $h_B$'s column-locked assumptions would over-penalise rotational states
and miss optimal trajectories.

\paragraph{Robotic Hull Inspection Case Study.}
An articulated manipulator arm begins in an unconstrained intake cavity
(\textbf{Class~A}), transitions to tight rib structures that require full
sequential clearance of all blocking components (\textbf{Class~D},
$h_D = d_v + \sum g_p$), then encounters localised protrusions where a
two-vacancy bypass is available (\textbf{Class~E},
$h_E = d_v + \sum \lceil g_p/2 \rceil - \delta$), and finally executes a
precision solder operation where the comb-shaped end-effector is body-pinned
and cannot flip (\textbf{Class~F}, $h_F = \lceil d/k \rceil + \sum \lceil
g_i/k \rceil$). Each class transition triggers a heuristic switch; using
Class~D's heuristic in the Class~E phase would over-estimate costs by
ignoring the bypass move available when $k=2$, violating the admissibility
requirement established for Class~E (Theorem~\ref{thm:hE_consistency}).

These case studies confirm that the CBHS framework is not a static
classification tool for isolated puzzle instances but a generalizable mechanism
for adaptive spatial planning under non-static physical constraints.

\section{Configurations Beyond Current Algorithmic Reach}\label{app:hard-configs}

Figure~\ref{fig:hard-configurations} illustrates several examples of challenging configurations and highlights their structural differences from typical solvable instances.

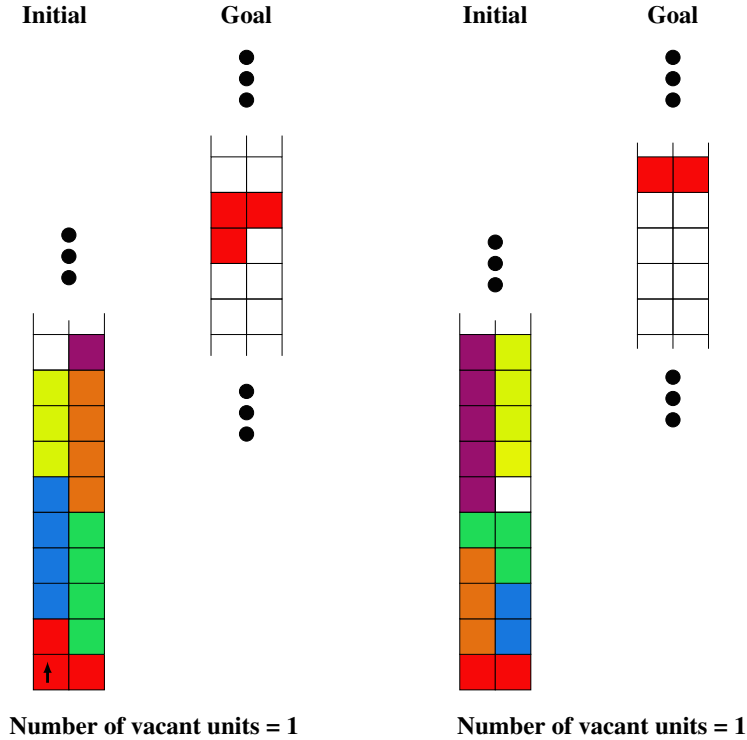
\begin{figure}[htbp]
	\centering
	\resizebox{0.6\textwidth}{!}{
		\begin{tikzpicture}

			\tikzstyle{every node}=[font=\Huge]
			\draw [ fill={rgb,255:red,247; green,8; blue,8} ] (-3.75,-1.25) rectangle (-2.5,-2.5);
			\draw [ fill={rgb,255:red,247; green,8; blue,8} ] (-1.25,-1.25) rectangle (-2.5,-2.5);
			\draw [ fill={rgb,255:red,247; green,8; blue,8} ] (-3.75,0) rectangle (-2.5,-1.25);
			\draw [ fill={rgb,255:red,31; green,219; blue,87} ] (-1.25,0) rectangle (-2.5,-1.25);
			\draw [ fill={rgb,255:red,23; green,119; blue,222} ] (-3.75,1.25) rectangle (-2.5,0);
			\draw [ fill={rgb,255:red,31; green,219; blue,87} ] (-1.25,1.25) rectangle (-2.5,0);
			\draw [ fill={rgb,255:red,23; green,119; blue,222} ] (-3.75,2.5) rectangle (-2.5,1.25);
			\draw [ fill={rgb,255:red,31; green,219; blue,87} ] (-1.25,2.5) rectangle (-2.5,1.25);
			\draw [ fill={rgb,255:red,23; green,119; blue,222} ] (-3.75,3.75) rectangle (-2.5,2.5);
			\draw [ fill={rgb,255:red,31; green,219; blue,87} ] (-1.25,3.75) rectangle (-2.5,2.5);
			\draw [ fill={rgb,255:red,23; green,119; blue,222} ] (-3.75,5) rectangle (-2.5,3.75);
			\draw [ fill={rgb,255:red,228; green,112; blue,17} ] (-1.25,5) rectangle (-2.5,3.75);
			\draw [ fill={rgb,255:red,216; green,244; blue,6} ] (-3.75,6.25) rectangle (-2.5,5);
			\draw [ fill={rgb,255:red,228; green,112; blue,17} ] (-1.25,6.25) rectangle (-2.5,5);
			\draw [ fill={rgb,255:red,216; green,244; blue,6} ] (-3.75,7.5) rectangle (-2.5,6.25);
			\draw [ fill={rgb,255:red,228; green,112; blue,17} ] (-1.25,7.5) rectangle (-2.5,6.25);
			\draw [ fill={rgb,255:red,216; green,244; blue,6} ] (-3.75,8.75) rectangle (-2.5,7.5);
			\draw [ fill={rgb,255:red,228; green,112; blue,17} ] (-1.25,8.75) rectangle (-2.5,7.5);
			\draw  (-3.75,10) rectangle (-2.5,8.75);
			\draw [ fill={rgb,255:red,152; green,16; blue,109} ] (-1.25,10) rectangle (-2.5,8.75);
			\draw (-1.25,10) -- (-1.25,10.75);
			\draw (-2.5,10) -- (-2.5,10.5);
			\draw (-3.75,10) -- (-3.75,10.75);

			\draw [ fill={rgb,255:red,247; green,8; blue,8} ] (3.75,13.75) rectangle (2.5,12.5);
			\draw  (3.75,12.5) rectangle (2.5,11.25);
			\draw  (5,12.5) rectangle (3.75,11.25);
			\draw  (5,13.75) rectangle (3.75,12.5);
			\draw [ fill={rgb,255:red,247; green,8; blue,8} ] (5,15) rectangle (3.75,13.75);
			\draw [ fill={rgb,255:red,247; green,8; blue,8} ] (2.5,15) rectangle (3.75,13.75);
			\draw  (5,16.25) rectangle (3.75,15);
			\draw  (2.5,16.25) rectangle (3.75,15);
			\draw  (5,11.25) rectangle (3.75,10);
			\draw  (2.5,11.25) rectangle (3.75,10);
			\draw (5,16.25) -- (5,17);
			\draw (3.75,16.25) -- (3.75,17);
			\draw (2.5,16.25) -- (2.5,17);
			\draw (5,10) -- (5,9.25);
			\draw (3.75,10) -- (3.75,9.25);
			\draw (2.5,10) -- (2.5,9.25);

			\draw [ fill={rgb,255:red,247; green,8; blue,8} ] (11.25,-1.25) rectangle (12.5,-2.5);
			\draw [ fill={rgb,255:red,247; green,8; blue,8} ] (13.75,-1.25) rectangle (12.5,-2.5);
			\draw [ fill={rgb,255:red,228; green,112; blue,17} ] (11.25,0) rectangle (12.5,-1.25);
			\draw [ fill={rgb,255:red,23; green,119; blue,222} ] (13.75,0) rectangle (12.5,-1.25);
			\draw [ fill={rgb,255:red,228; green,112; blue,17} ] (11.25,1.25) rectangle (12.5,0);
			\draw [ fill={rgb,255:red,23; green,119; blue,222} ] (13.75,1.25) rectangle (12.5,0);
			\draw [ fill={rgb,255:red,228; green,112; blue,17} ] (11.25,2.5) rectangle (12.5,1.25);
			\draw [ fill={rgb,255:red,31; green,219; blue,87} ] (13.75,2.5) rectangle (12.5,1.25);
			\draw [ fill={rgb,255:red,31; green,219; blue,87} ] (11.25,3.75) rectangle (12.5,2.5);
			\draw [ fill={rgb,255:red,31; green,219; blue,87} ] (13.75,3.75) rectangle (12.5,2.5);
			\draw [ fill={rgb,255:red,152; green,16; blue,109} ] (11.25,5) rectangle (12.5,3.75);
			\draw  (13.75,5) rectangle (12.5,3.75);
			\draw [ fill={rgb,255:red,152; green,16; blue,109} ] (11.25,6.25) rectangle (12.5,5);
			\draw [ fill={rgb,255:red,228; green,244; blue,6} ] (13.75,6.25) rectangle (12.5,5);
			\draw [ fill={rgb,255:red,152; green,16; blue,109} ] (11.25,7.5) rectangle (12.5,6.25);
			\draw [ fill={rgb,255:red,216; green,244; blue,6} ] (13.75,7.5) rectangle (12.5,6.25);
			\draw [ fill={rgb,255:red,152; green,16; blue,109} ] (11.25,8.75) rectangle (12.5,7.5);
			\draw [ fill={rgb,255:red,216; green,244; blue,6} ] (13.75,8.75) rectangle (12.5,7.5);
			\draw [ fill={rgb,255:red,152; green,16; blue,109} ] (11.25,10) rectangle (12.5,8.75);
			\draw [ fill={rgb,255:red,216; green,244; blue,6} ] (13.75,10) rectangle (12.5,8.75);
			\draw (13.75,10) -- (13.75,10.75);
			\draw (12.5,10) -- (12.5,10.5);
			\draw (11.25,10) -- (11.25,10.75);

			\draw  (18.75,11.25) rectangle (17.5,10);
			\draw  (20,11.25) rectangle (18.75,10);
			\draw  (20,12.5) rectangle (18.75,11.25);
			\draw  (17.5,12.5) rectangle (18.75,11.25);
			\draw  (20,13.75) rectangle (18.75,12.5);
			\draw  (17.5,13.75) rectangle (18.75,12.5);
			\draw  (20,15) rectangle (18.75,13.75);
			\draw  (17.5,15) rectangle (18.75,13.75);
			\draw [ fill={rgb,255:red,247; green,8; blue,8} ] (20,16.25) rectangle (18.75,15);
			\draw [ fill={rgb,255:red,247; green,8; blue,8} ] (17.5,16.25) rectangle (18.75,15);
			\draw (20,10) -- (20,9.5);
			\draw (18.75,10) -- (18.75,9.5);
			\draw (17.5,10) -- (17.5,9.5);
			\draw (20,16.25) -- (20,16.75);
			\draw (18.75,16.25) -- (18.75,16.75);
			\draw (17.5,16.25) -- (17.5,16.75);

			\draw [ fill={rgb,255:red,3; green,3; blue,3} ] (-2.5,12) circle (0.25cm);
			\draw [ fill={rgb,255:red,3; green,3; blue,3} ] (-2.5,12.75) circle (0.25cm);
			\draw [ fill={rgb,255:red,3; green,3; blue,3} ] (-2.5,13.5) circle (0.25cm);
			\draw [ fill={rgb,255:red,3; green,3; blue,3} ] (3.75,18.25) circle (0.25cm);
			\draw [ fill={rgb,255:red,3; green,3; blue,3} ] (3.75,19) circle (0.25cm);
			\draw [ fill={rgb,255:red,3; green,3; blue,3} ] (3.75,19.75) circle (0.25cm);
			\draw [ fill={rgb,255:red,3; green,3; blue,3} ] (3.75,6.5) circle (0.25cm);
			\draw [ fill={rgb,255:red,3; green,3; blue,3} ] (3.75,7.25) circle (0.25cm);
			\draw [ fill={rgb,255:red,3; green,3; blue,3} ] (3.75,8) circle (0.25cm);
			\draw [ fill={rgb,255:red,3; green,3; blue,3} ] (12.5,11.75) circle (0.25cm);
			\draw [ fill={rgb,255:red,3; green,3; blue,3} ] (12.5,12.5) circle (0.25cm);
			\draw [ fill={rgb,255:red,3; green,3; blue,3} ] (12.5,13.25) circle (0.25cm);
			\draw [ fill={rgb,255:red,3; green,3; blue,3} ] (18.75,18.25) circle (0.25cm);
			\draw [ fill={rgb,255:red,3; green,3; blue,3} ] (18.75,19) circle (0.25cm);
			\draw [ fill={rgb,255:red,3; green,3; blue,3} ] (18.75,19.75) circle (0.25cm);
			\draw [ fill={rgb,255:red,3; green,3; blue,3} ] (18.75,7) circle (0.25cm);
			\draw [ fill={rgb,255:red,3; green,3; blue,3} ] (18.75,7.75) circle (0.25cm);
			\draw [ fill={rgb,255:red,3; green,3; blue,3} ] (18.75,8.5) circle (0.25cm);

			\node [font=\Huge] at (-3,21.25) {\textbf{Initial}};
			\node [font=\Huge] at (3.75,21.25) {\textbf{Goal}};
			\node [font=\Huge] at (12.5,21.25) {\textbf{Initial}};
			\node [font=\Huge] at (18.75,21.25) {\textbf{Goal}};
			\node [font=\Huge] at (0.5,-3.75) {\textbf{Number of vacant units = 1}};
			\node [font=\Huge] at (16.25,-3.75) {\textbf{Number of vacant units = 1}};
			\draw [->, line width=1mm, >=latex] (-3.25,-2.25) -- (-3.25,-1.5);

		\end{tikzpicture}
	}
	\caption{Some examples of challenging configurations beyond the reach of our algorithms.}
	\label{fig:hard-configurations}
\end{figure}

\end{document}